\documentclass{article}

\PassOptionsToPackage{numbers, compress}{natbib}
\usepackage[final]{neurips_2022}

\usepackage[utf8]{inputenc} % allow utf-8 input
\usepackage[T1]{fontenc}    % use 8-bit T1 fonts
\usepackage{hyperref}       % hyperlinks
\usepackage{url}            % simple URL typesetting
\usepackage{booktabs}       % professional-quality tables
\usepackage{amsfonts}       % blackboard math symbols
\usepackage{nicefrac}       % compact symbols for 1/2, etc.
\usepackage{microtype}      % microtypography
\usepackage{xcolor}         % colors

\usepackage{multirow}
\usepackage{booktabs}
\usepackage{graphicx}
\usepackage{makecell}
\usepackage{tabu}
\usepackage{collectbox}
\usepackage{scalerel}
\usepackage{bbm}
\usepackage{amsmath,amsfonts,bm,amsthm,amssymb,mathrsfs,algorithm,algorithmic,xcolor,url}
\usepackage{booktabs,enumitem,multicol,caption,subcaption,pifont,tikz,listings,wrapfig,lipsum,circledsteps,arydshln}
\usepackage{relsize}

\hypersetup{
    colorlinks,
    linkcolor={blue!50!black},
    citecolor={blue!50!black},
    urlcolor={blue!50!black}
}

\newcommand{\circled}[1]{\textcircled{\raisebox{-0.9pt}{#1}}}

\theoremstyle{plain}
\newtheorem{theorem}{Theorem} %[section]
\newtheorem{proposition}[theorem]{Proposition}
\newtheorem{lemma}[theorem]{Lemma}

\theoremstyle{definition}

\theoremstyle{remark}
\newtheorem{remark}[theorem]{Remark}
\newtheorem*{theorem*}{Theorem}

\def\Secref#1{Section~\ref{#1}}
\def\eqref#1{(\ref{#1})}
\def\1{\bm{1}}

\def\vzero{{\bm{0}}}

\def\vtheta{{\bm{\theta}}}

\def\vphi{{\bm{\phi}}}

\def\vpsi{{\bm{\psi}}}

\def\mI{{\bm{I}}}

\def\mT{{\bm{T}}}

\DeclareMathAlphabet{\mathsfit}{\encodingdefault}{\sfdefault}{m}{sl}
\SetMathAlphabet{\mathsfit}{bold}{\encodingdefault}{\sfdefault}{bx}{n}

\def\gB{{\mathcal{B}}}

\def\gD{{\mathcal{D}}}

\def\gL{{\mathcal{L}}}
\def\gM{{\mathcal{M}}}
\def\gN{{\mathcal{N}}}

\def\gP{{\mathcal{P}}}

\def\gT{{\mathcal{T}}}

\def\gV{{\mathcal{V}}}

\def\sA{{\mathbb{A}}}

\def\sS{{\mathbb{S}}}

\newcommand{\E}{\mathbb{E}}

\newcommand{\R}{\mathbb{R}}

\newcommand{\abs}[1]{\left\vert#1\right\rvert}

\newcommand{\cbr}[1]{\left\{#1\right\}}
\newcommand{\br}[1]{\left(#1\right)}
\newcommand{\sbr}[1]{\left[#1\right]}
\newcommand{\given}{\,|\,}

\newcommand{\denv}{\gD_{\mathrm{env}}}
\newcommand{\dmodel}{\gD_{\mathrm{model}}}

\newcommand{\dbgtrue}{d_{\pi_b,\gamma}^{P^{*}}}
\newcommand{\dpigtrue}{d_{\pi,\gamma}^{P^{*}}}

\newcommand{\rmax}{r_{\max}}

\newcommand*\diff{\mathop{}\!\mathrm{d}}

\newcommand{\ie}{\textit{i.e.}}
\newcommand{\eg}{\textit{e.g.}}
\newcommand{\wrt}{\textit{w.r.t.}}
\newcommand{\etc}{\textit{etc.}}
\newcommand{\myparagraph}[1]{\textbf{#1~~~~}}

\title{
A Unified Framework for Alternating Offline \\ Model Training and Policy Learning
}

\author{Shentao Yang\textsuperscript{1},~~~~Shujian Zhang\textsuperscript{1},~~~~Yihao Feng\textsuperscript{2},~~~~Mingyuan Zhou\textsuperscript{1}
\\
\textsuperscript{1}The University of Texas at Austin~~~~
\textsuperscript{2}Salesforce Research~~~~
\\
\texttt{shentao.yang@mccombs.utexas.edu}~~~~~~\quad \texttt{szhang19@utexas.edu}\\
\texttt{yihaof@salesforce.com}\quad \texttt{mingyuan.zhou@mccombs.utexas.edu}
}

\begin{document}

\maketitle

\begin{abstract}
In offline model-based reinforcement learning (offline MBRL), we learn a dynamic model from historically collected data, and subsequently utilize the learned model and fixed datasets for policy learning, without further interacting with the environment. Offline MBRL algorithms can improve the efficiency and stability of policy learning over the model-free algorithms. However, in most of the existing offline MBRL algorithms, the learning objectives for the dynamic models and the policies are isolated from each other. Such an \emph{objective mismatch} may lead to inferior performance of the learned agents. In this paper, we address this issue by developing an iterative offline MBRL framework, where we maximize a lower bound of the true expected return, by alternating between dynamic-model training and policy learning. With the proposed unified model-policy learning framework, we achieve competitive performance on a wide range of continuous-control offline reinforcement learning datasets.
Source code is publicly \href{https://github.com/Shentao-YANG/AMPL_NeurIPS2022}{released}.
\end{abstract}

\section{Introduction}\label{sec:intro}

% offline rl with general benefit, MBRL with general benefit
Offline reinforcement learning (offline RL) \citep{batchrl2012,offlinetutorial2020}, where the agents are trained from static and pre-collected datasets, avoids direct interactions with the underlying real-world environment during the learning process.
Unlike traditional online RL,
%unlike traditional online RL, avoids interactions with the real underlying environment during the training process, and instead learns policy purely from static datasets \citep{batchrl2012, offlinetutorial2020}.
whose success largely depends on simulator-based trial-and-error \citep[\eg,][]{dqn2015,alphago2016}, 
offline RL enables training policies for real-world applications, where it is infeasible or even risky to collect online experimental data, such as robotics, advertisement, or dialog systems \citep[\eg,][]{li11unbiased,autodrivesurvey2020,gptcritic2022,yang2022dynamic}.
Though promising, 
it remains challenging to train agents under the offline setting, due to the discrepancy between the distribution of the offline data and the state-action distribution induced by the current learning policy. 
With such a discrepancy, directly transferring standard online off-policy RL methods \citep[\eg,][]{ddpg2016,sac2018,td32018} to the offline setting tends to be problematic \citep{bcq2019,bear2019}, especially when the offline data cannot sufficiently cover the state-action space \citep{rem2020}.
To tackle this issue in offline RL, 
recent works \citep[\eg,][]{mbpo2019,mopo2020,gamembrl2020} propose to approximate the policy-induced state-action distribution by leveraging a learned dynamic model to draw imaginary rollouts. 
These additional synthetic rollouts help mitigate the distributional discrepancy and 
stabilize the policy-learning algorithms under the offline setting.
%Model-based RL (MBRL) appears as a natural solution to this difficulty of offline RL.
%The intuition is that with a good dynamic model, one can approximately sample from the stationary state-action distribution induced by the learned policy via rolling out the learned policy on the learned model.
%Combining these synthetic data with the offline dataset can ideally mitigate the stated distributional discrepancy \citep{mbpo2019,gamembrl2020,mopo2020}.

Most of the prior offline model-based RL (MBRL) methods \citep[\eg,][]{mopo2020,combo2021,morel2020,moose2021,mabe2021},
% , such as MOPO \citep{mopo2020}, MoReL \citep{morel2020}, and COMBO \citep{combo2021}, 
however, first \emph{pretrain} a one-step forward dynamic model via maximum likelihood estimation (MLE) on the offline dataset,
and then use the learned model to train the policy, without \emph{further improving the dynamic model} during the policy learning process.
As a result, the objective function used for model training (\eg, MLE) and the objective of model utilization are \emph{unrelated} with each other. 
Specifically, the model is trained to be ``simply a mimic of the world,'' but is used to improve the performance of the learned policy \citep{slbo2019,objectivemismatch2020,mnm2021}.
Though such a training paradigm is historically rooted \citep{dyna1991,Bertsekas1995DynamicPA},
this issue of \emph{objective mismatch} in the model training and  model utilization has been identified as problematic in recent works \citep{objectivemismatch2020,mnm2021}.
In offline MBRL, this issue is exacerbated, since the learned model can hardly be globally accurate, due to the limited amount of offline data and the complexity of the control tasks. 

Motivated by the objective-mismatch issue, 
we develop an iterative offline MBRL method, 
alternating between training the dynamic model and %learning
the policy to maximize a lower bound of the true expected return. 
This lower bound, leading to a weighted MLE objective for the dynamic-model training, is relaxed to a tractable regularized objective for the policy learning.
To train the dynamic model by the proposed objective, 
we need to estimate the marginal importance weights (MIW) between the offline-data distribution and the stationary state-action distribution of the current policy \citep{breakingcurse2018,blackope2020}.
This estimation tends to be unstable by standard approaches \citep[\eg,][]{dualdice2019,gendice2020}, which require saddle-point optimization.
Instead, we propose a simple yet stable fixed-point-style method for MIW estimation, which can be directly incorporated into our alternating training framework. 
With these considerations, our method, offline Alternating Model-Policy Learning (AMPL), performs competitively on a wide range of continuous-control offline RL datasets in the D4RL benchmark \citep{fu2021d4rl}. 
These empirical results and ablation study show the efficacy of our proposed algorithmic designs.

%Besides, to stabilize the estimation of marginal importance weights \citep{breakingcurse2018,gendice2020} used in model training, we further derive a simple marginal importance weights training approach, which avoids solving the saddle-point optimization in the previous methods  

\iffalse
In this paper, we are motivated by the objective-mismatch issue to develop an iterative offline MBRL method that alternates between training the model to minimize the evaluation error of the learned policy, and training the policy to maximize a lower bound of the true expected return. 
Specifically, we derive a tractable upper bound for the policy evaluation error that can be used for both model training and policy learning.
This bound leads to a weighed MLE objective for model training and a weighted regularizer for policy training, where the weight is the marginal importance weight (MIW) in the off-policy evaluation (OPE) literature.
We derive an iterative method for estimating the MIW, to avoid the instability in the common methods that involves saddle-point optimization.
% relax regularizer
To enhance training stability, we upper bound the weighted regularizer further to derive a stronger regularization for policy training.
With these considerations, our method shows competitive performance on a wide range of continuous-control offline RL datasets in the D4RL benchmark \citep{fu2021d4rl}.
Together with the ablation study, this validates the efficacy of our algorithmic designs. 
\fi 

\section{Background} \label{sec:background}

% RL
\textbf{Markov decision process and offline RL.~~~~} A Markov decision process (MDP) is denoted by $\gM = \br{\sS, \sA, P, r, \gamma, \mu_0}$, where $\sS$ is the state space, $\sA$ the action space, $P\br{s' \given s, a}: \sS \times \sS \times \sA \rightarrow [0,1]$ the environmental dynamic, $r(s,a): \sS \times \sA \rightarrow [-r_{max}, r_{max}]$ the reward function, $\gamma \in [0,1)$ the discount factor, and $\mu_0(s): \sS \rightarrow [0,1]$ the initial state-distribution.

% OPE (def of both discounted and undiscounted distribution)
For any policy $\pi(a\given s)$, we denote its state-action distribution at timestep $t \geq 0$ as 
$
 d_{\pi, t}^P(s,a) \triangleq  \Pr\br{s_t = s, a_t = a \given s_0 \sim \mu_0, a_t \sim \pi, s_{t+1} \sim P,~\forall\, t \geq 0}.
$
The (discounted) stationary state-action distribution of $\pi$ is denoted as
$
    d_{\pi,\gamma}^P(s,a) \triangleq (1 - \gamma) \sum_{t=0}^\infty \gamma^t d_{\pi, t}^P(s,a).
$
% and the \textit{undiscounted (average)} stationary state-action distribution as $d_\pi^P(s,a) = \lim_{T \rightarrow \infty}\frac{1}{T+1}\sum_{t=0}^T  d_{\pi, t}^P(s,a)$
%We have $d_{\pi,\gamma}^P(s,a) = d_{\pi,\gamma}^P(s) \pi(a \given s)$.
% and $d_\pi^P(s,a) = d_\pi^P(s) \pi(a \given s)$.
% offline RL background
%In offline RL \citep{offlinetutorial2020}, we only have access to a static dataset $\denv = \cbr{(s_{i},a_{i},r_{i},s^{\prime}_i)}_{i=1}^{n} \sim d_{\pi_b, \gamma}^P$ collected by some behavior policy $\pi_b$, which can be a mixture of several data-collecting policies.

% Actor-Critic Algorithm (def of Q, training of Q, training of \pi)

Denote $Q_\pi^P(s,a) = \E_{\pi, P}\sbr{\sum_{t=0}^\infty \gamma^t r(s_t,a_t) \given s_0 = s, a_0 = a}$ as the action-value function of policy $\pi$ under the dynamic $P$.
The goal of RL is to find a policy $\pi$ maximizing the expected return
\begin{equation} \textstyle \label{eq:rl_goal}
    J(\pi,P) \triangleq
 (1-\gamma)\mathbb{E}_{s\sim\mu_{0},a\sim\pi(\cdot \given s)}\left[Q_{\pi}^{P}(s,a)\right]
 =\mathbb{E}_{(s,a)\sim d_{\pi,\gamma}^{P}}\left[r(s,a)\right].
\end{equation}
In offline RL, the policy $\pi$ and critic $Q_\pi^{P}$ are typically approximated by parametric functions $\pi_\vphi$ and $Q_\vtheta$, respectively, with parameters $\vphi$ and $\vtheta$.
The critic $Q_\vtheta$ is trained by the Bellman backup
\begin{equation}\textstyle \label{eq:ac_bellman}
    \arg\min_{\vtheta} \E_{(s,a,r,s')\sim \denv} \sbr{\br{Q_{\vtheta}(s, a) - \br{ r(s, a) + \gamma \E_{\, a' \sim \pi_\vphi(\cdot \given s')} \sbr{Q_{\vtheta'}(s', a')}}}^2},
\end{equation}
where $Q_{\vtheta^\prime}$ is the target network \citep{bcq2019,bear2019}. 
The actor $\pi_\vphi$ is trained in the policy improvement step by
\begin{equation}\textstyle \label{eq:ac_actor}
     \arg\max_\vphi \E_{s \sim \denv,\, a \sim \pi_{\vphi}(\cdot \given s)}\sbr{Q_{\vtheta}\br{s, a}},
\end{equation}
where $\denv$ denotes the offline dataset drawn from $d_{\pi_b,\gamma}^P$ \citep{offlinetutorial2020,diagbottleneck2019}, with $\pi_b$ being the behavior policy.

% model-based RL
\textbf{Offline model-based RL.~~~~} In offline model-based RL algorithms, the true environmental dynamic $P^*$ is typically approximated by a parametric function $\widehat P\br{s' \given s, a}$ in some function class $\mathcal{P}$.
With the offline dataset $\denv$, $\widehat P$ is trained via the MLE \citep{mbpo2019, mopo2020, combo2021}  as
\begin{equation}\textstyle \label{eq:model_objective}
    \arg\max_{\widehat P \in \mathcal{P}} \E_{\br{s,a,s'} \sim \denv}\sbr{\log \widehat P\br{s' \given s, a}}.
\end{equation}
Similarly, the reward function can be approximated by a parametric model $\widehat r$ if assumed unknown.
With $\widehat P$ and $\widehat r$, the true MDP $\gM$ can be approximated by $\widehat \gM = (\sS, \sA, \widehat P, \widehat r, \gamma, \mu_0)$.
% We then have $d_\pi^{P^*}(s,a)$ as the undiscounted stationary state-action distribution induced by $\pi$ on $P^*$ (or $\gM$), $d_\pi^{\widehat P}(s,a)$ as on the learned dynamics $\widehat P$ (or $\widehat M$); 
%
We further define $d_{\pi,\gamma}^{P^*}(s,a)$ as the stationary state-action distribution induced by $\pi$ on $P^*$ (or MDP $\gM$), and $d_{\pi,\gamma}^{\widehat P}(s,a)$ as that on the learned dynamic $\widehat P$ (or MDP $\widehat \gM$).
% , the discounted stationary state-action distributions on $P^*$ and $\widehat P$.
We approximate $d_{\pi_\vphi,\gamma}^{P^*}$ by simulating $\pi_\vphi$ on $\widehat \gM$ for a short horizon $h$ starting from state $s \in \denv$, as in prior work \citep[\eg,][]{mopo2020,combo2021,morel2020,mabe2021}.
The resulting transitions are stored in a replay buffer $\dmodel$, constructed similar to the off-policy RL \citep{replaybuffer1992,ddpg2016}.
To better approximate $d_{\pi_\vphi,\gamma}^{P^*}$, sampling from $\denv$ in Eqs.~\eqref{eq:ac_bellman} and \eqref{eq:ac_actor} is commonly replaced by sampling from the augmented dataset $\mathcal{D} = f \denv + (1-f) \dmodel, f \in [0,1]$, denoting  sampling from $\denv$ and $\dmodel$ with probabilities $f$ and $1 - f$, respectively.
We follow \citet{combo2021} to use $f=0.5$.

% We follow the literature \citep{breakingcurse2018, algaedice2019, confoundingrobust2020,gendice2020} to assume that
% \textbf{(i)} any undiscounted stationary state-action distribution $d_\pi^P(s,a)$ equals to the state-action visitation frequency induced by $\pi$ on the corresponding MDP, \ie, all the Markov chains induced by the (approximated) environmental dynamics and policies are ergodic;
% and \textbf{(ii)}  $\denv \sim d_{\pi_b}^{P^*}(s,a)$, \ie, the offline dataset $\denv$ consists of rollouts of $\pi_b$ on the true dynamics $P^*$.
% We denote the state-action and state distributions in $\denv$ as $d_{\denv}(s,a)$ and $d_{\denv}(s)$, which are  discrete approximations to $d_{\pi_b}^{P ^*}(s,a)$ and $d_{\pi_b}^{P ^*}(s)$, respectively.

\section{Offline alternating model-policy learning} \label{sec:main_method}
Our goal is to derive the objectives for both dynamic-model training and policy learning from a principled perspective.
A natural idea is to build a tractable lower bound for $J(\pi, P^{*})$,
the expected return of the policy $\pi$ under the true dynamic $P^{*}$,
and then alternate between training the policy $\pi$ and the dynamic model $\widehat P$ to maximize this lower bound.
% Assume the reward function $r(s,a)$ is known, then we can immediately construct the following lower bound for $J(\pi, P^{*})$:
Indeed, we can construct a lower bound as % for $J(\pi, P^{*})$:
\begin{equation}\textstyle
    J(\pi,P^{*}) \geq J(\pi, \widehat{P}) - \vert J(\pi,P^{*}) - J(\pi, \widehat{P}) \vert \,, \label{equ:lower_bound}
\end{equation}
where $J(\pi, \widehat{P})$ is the expected return of  policy $\pi$ under the learned model $\widehat P$.
From the right hand side (RHS) of  Eq.~\eqref{equ:lower_bound}, if the policy evaluation error $|J(\pi,P^{*}) - J(\pi, \widehat{P})|$ is small, $J(\pi, \widehat P )$ will be a good proxy for the true expected return $J(\pi, P^{*})$. 
We can empirically estimate $J(\pi, \widehat P )$ via %the dynamic 
 $\widehat P$ and $\pi$. 

Further, if a tractable upper bound for $|J(\pi,P^{*}) - J(\pi, \widehat{P})|$ can be constructed, 
it can serve as a unified training objective for both dynamic model $\widehat P$ and policy  $\pi$. 
We can then alternate between optimizing the dynamic model $\widehat P$ and the policy $\pi$ 
to maximize the lower bound of $J(\pi, P^{*})$, \ie, simultaneously minimizing the upper bound of the evaluation error $|J(\pi,P^{*}) - J(\pi, \widehat{P})|$.
This gives us an \emph{iterative, maximization-maximization} algorithm for model and policy learning.

%\st{I don't think we need ``maximization-maximization''. I think the last sentense of this paragraph is not needed.}

%We can thus instead find a policy $\pi$ maximizing the right hand side (RHS) and train the model $\widehat{P}$ to minimize the evaluation error of $\pi$, $|J(\pi,P^{*}) - J(\pi, \widehat{P})|$.
%However, $\pi$ changes during the policy learning process and thus the $\widehat{P}$ that minimizes the policy evaluation error of $\pi$ is also fluid.
%We thus retrain the model $\widehat{P}$ periodically as the policy learning proceeds.

%\subsection{Model and policy training} \label{sec:main_method_model_policy}

The following theorem indicates a tractable upper bound for $|J(\pi,P^{*}) - J(\pi, \widehat{P})|$, 
which can be subsequently relaxed for model training and policy learning.

\begin{theorem}\label{thm:eval_error_upper_bound}
Let $P^{*}$ be the true dynamic and $\widehat P$ be the approximate dynamic model. Suppose the reward function $| r(s,a) |\leq r_{\max}$, then we have 
\begin{align*}
  & \left\vert J(\pi,P^{*})-J(\pi,\widehat{P})\right\vert
 \leq \frac{\gamma \cdot \rmax}{\sqrt{2}(1-\gamma)} \cdot \sqrt{ D_{\pi} (P^{*}, \widehat P)} \,,\\
 \text{with}~~~~D_{\pi} (P^{*}, \widehat P) &\triangleq \E_{(s,a) \sim \dbgtrue}\sbr{\omega(s,a) \mathrm{KL}\br{P^*(s'\given s,a) \pi_b(a'\given s') \,||\, \widehat P(s' \given s,a)\pi(a'\given s') }} \,,
\end{align*}
where $\pi_b$ is the behavior policy, $d_{\pi_b ,\gamma}^{P^{*}}$  is the offline-data distribution, and $\omega(s,a) \triangleq \frac{d_{\pi,\gamma}^{P^{*}}(s,a)}{\dbgtrue(s,a)}$
is the marginal importance weight (MIW) between the offline-data distribution and the stationary state-action distribution of the policy $\pi$ \citep{breakingcurse2018,dualdice2019}. 
%\st{do we need this which...?} which is introduced in recent density ratio based off-policy evaluation techniques \citep[e.g.,][]{breakingcurse2018,dualdice2019}.
%MIW, which can be obtained by techniques
%from off-policy evaluation (OPE, \Secref{sec:main_method_miw}).
\end{theorem}
% \begin{equation*}\textstyle \label{eq:eval_error_policy_model_main}
% \resizebox{1.\textwidth}{!}{%
%     $
    
% $%
% }
% \end{equation*}
Detailed proof of Theorem \ref{thm:eval_error_upper_bound} can be found in Appendix~\ref{sec:proof_eval_error_upper_bound}.

% If we take a closer look at the KL term in $D_{\pi}(P^{*}, \widehat P)$, intuitively it indicates the following two training principles for model and policy learning:

The KL term in $D_{\pi}(P^{*}, \widehat P)$ indicates the following two principles for 
model and policy learning:
%
%$\bullet~~$ 

$\star~~$ For the dynamic model, $s^\prime \sim \widehat P(\cdot \given s,a )$ should be close to the true next state $\tilde{s}^\prime \sim P^{*}(\cdot \given s, a)$, with $(s,a)$ pairs drawn from the stationary state-action distribution of the policy $\pi$. Since we cannot directly draw samples from $d_{\pi, \gamma}^{P^{*}}$, we reweight the offline data with  $\omega(s,a)$.
This leads to a weighted KL minimization objective for the model training.

$\star~~$ For the policy $\pi$, the KL term indicates a regularization term, that the tuple $(s^\prime ,a^\prime)$ from the joint conditional distribution $\widehat P(s' \given s,a)\pi(a'\given s')$ should be close to the tuple $(\tilde{s}^{\prime}, \tilde{a}^{\prime})$ from $P^*(\tilde{s}^\prime\given s,a) \pi_b(\tilde{a}^{\prime}\given \tilde s^{\prime})$.  $(\tilde{s}^{\prime}, \tilde{a}^{\prime})$ is simply a sample from the offline dataset.

Based on the above observations, 
we can fixed $\pi$ and train the dynamic model $\widehat P$ by minimizing $D_{\pi}(P^{*}, \widehat P)$ \wrt~$\widehat P$.
Similarly, we can fix the dynamic model $\widehat P$ and learn a policy $\pi$ to maximize the lower bound of $J(\pi, P^*)$.
This alternating training scheme provides a unified approach for model and policy learning. 
In the following sections, 
we discuss how to optimize the dynamic model $\widehat P$, the policy $\pi$ , and the MIW $\omega$ under our alternating training farmework.

% The RHS of Theorem~\ref{thm:eval_error_upper_bound} is tractable since it only involves samples from the offline dataset, the MIW $\omega$, the current model $\widehat P$, and the current policy $\pi$.
% With the estimated $\omega$, we can fix the policy $\pi$ to minimize $\gL(\pi, \widehat P)$ \wrt~the model $\widehat P$ and fix $\widehat P$ to minimize $\gL(\pi, \widehat P)$ \wrt~$\pi$.
% This suggests an iterative method that alternates between estimating $\omega, \widehat P$ and optimizing $\pi$.
% \textbf{Model training objective.}
% Expanding the $\mathrm{KL}$ term in $\gL(\pi, \widehat P)$ of Theorem~\ref{thm:eval_error_upper_bound}, we have

\subsection{Dynamic model training}
Expanding the KL term in $D_{\pi}(P^{*}, \widehat P)$, we have 
\begin{align*}
    D_{\pi}(P^{*}, \widehat P)= &\overbrace{\mathbb{E}_{(s,a, s^{\prime}, a^{\prime}) \sim \dbgtrue } \left[\omega(s,a) \br{\log P^{*}(s' \given s,a) + \log\pi_b(a^{\prime} \given s^{\prime}) - \log \pi\br{a' \given s'}} \right]}^{\triangleq~\Circled{1}} \\
    &~~~~~~~~~~~~~~~~~ - \mathbb{E}_{(s,a,s')\sim \dbgtrue}\left[ \omega(s,a) \log \widehat{P}(s' \given s,a) \right]\,,
\end{align*}
where the tuple $(s, a, s^{\prime}, a^{\prime})$ is simply two consecutive state-action pairs in the offline dataset.
% \begin{equation*}
% \resizebox{1.\textwidth}{!}{%
%     $
%     \begin{aligned}
% D_{\pi}(P^{*}, \widehat P)= \underbrace{\mathbb{E}_{\br{\substack{s,a,\\s',a'}}\sim \dbgtrue} \left[\omega(s,a) \br{\log \br{P^{*}(s' |s,a)\pi_b(a'| s')} - \log \pi\br{a' | s'}} \right]}_{\triangleq~\Circled{1}} - \mathbb{E}_{(s,a,s')\sim \dbgtrue}\left[ \omega(s,a) \log \widehat{P}(s' | s,a) \right].  
% % \\ \triangleq & C \ell_\omega\left(\widehat{P}\right)
% \end{aligned}
% $%
% }
% \end{equation*}
Further, if the policy $\pi$ is fixed, the term \Circled{1} is a constant \wrt~$\widehat{P}$. Thus, given the MIW $\omega$, we can optimize $\widehat P$ by minimizing the following loss
\begin{equation} \label{eq:model_weighted_mle} \textstyle
      \ell(\widehat P) \triangleq - \mathbb{E}_{(s,a,s')\sim \dbgtrue}\left[ \omega(s,a) \log \widehat{P}(s' \given s,a) \right],
\end{equation}
which is an MLE objective weighted by $\omega(s,a)$. 
We discuss how to estimate $\omega(s,a)$ in \Secref{sec:main_method_miw}.

\subsection{Policy learning} \label{sec:policy_learning}
The lower bound for $J(\pi, P^{*})$ implied by Theorem~\ref{thm:eval_error_upper_bound} is
\begin{equation} \textstyle
    J(\pi, \widehat P) - \frac{\gamma \cdot \rmax}{\sqrt{2}(1-\gamma)} \cdot \sqrt{ D_{\pi} (P^{*}, \widehat P)}\,, \label{equ:new_lower_bound}
\end{equation}
where $J(\pi, \widehat P)$ can be estimated via the action-value function similar to standard offline MBRL algorithms \citep[\eg,][]{mopo2020,combo2021,morel2020}.
Thus, when the dynamic model $\widehat P$ is fixed, 
the main difficulty is to estimate the regularizer $D_{\pi}(P^{*}, \widehat P)$ for the policy $\pi$.

When the policy $\pi$ is Gaussian, direct estimation of $D_{\pi}(P^{*}, \widehat P)$ is possible.
Empirically, however, it is helpful to learn the policy $\pi$ in the class of implicit distribution, which is a richer distribution class and can better maximize the action-value function. 
Specifically, given a noise distribution $p_z(z)$, action $a = \pi_\vphi(s, z)$ with $z \sim p_z(\cdot)$, where $\pi_\vphi$ is a deterministic network.

Unfortunately, we can not directly estimate the KL term in $D_{\pi}(P^{*}, \widehat P)$ if $\pi$ is an implicit policy, since we can only draw samples from $\pi$ but the density is unknown.
A potential solution is to use the dual representation of KL divergence $\mathrm{KL}(p \,||\, q) = \sup_{T} \E_{p}[T] - \log(\E_q[e^T])$ \citep{kldual1983}, which can be estimated with samples from the distributions $p$ and $q$.
However, the exponential function therein makes the estimation unstable in practice \citep{gendice2020}.
We instead use the dual representation of the Jensen–Shannon divergence (JSD) to approximate $D_{\pi}(P^{*}, \widehat P)$, which can be approximately minimized using the GAN structure \citep{gan2014, liu2021fusedream}.
Our framework can thus utilize the many stabilization techniques developed in the GAN community (Appendix~\ref{sec:algo_details_gan}).

Besides, we remove the MIW $\omega(s,a)$ during the policy training since we do not observe its empirical benefits, which will be discussed in Section~\ref{sec:exp_ablation}. 
Further applying the replacement of KL with JSD and ignoring the $\sqrt{\cdot}$ for numerical stability, we get an \textit{approximated new regularization} for policy $\pi$:
\begin{equation} \label{eq:approx_new_reg}
    \resizebox{0.94\textwidth}{!}{%
    $
    \widetilde{D}_{\pi}(P^{*}, \widehat P) \triangleq \mathrm{JSD}\br{P^*(s'\given s,a) \pi_b(a'\given s')\dbgtrue(s,a) \,||\, \widehat P(s' \given s,a)\pi(a'\given s') \dbgtrue(s) \pi(a\given s)}\,.
$%
}
\end{equation}
% Intuitively, samples from the distribution in the left part of JSD are state-action pairs from the offline dataset.
% Samples from the distribution in the right part of JSD can be viewed as imaginary rollouts of $\pi$ on $\widehat P$. 
% As a result, the regularization term can implicitly enforce the synthetic augmented data distribution is close to the distribution offline dataset. 
Informally speaking, Eq.~\eqref{eq:approx_new_reg} regularizes the imaginary rollouts of $\pi$ on $\widehat P$ towards state-action pairs from the offline dataset.
Intuitively, $\widetilde{D}_{\pi}(P^{*}, \widehat P)$ is a more effective regularizer for policy training than the original ${D}_{\pi}(P^{*}, \widehat P)$, since $\widetilde{D}_{\pi}(P^{*}, \widehat P)$ regularizes action choices at both $s$ and $s'$.
Appendix~\ref{sec:derive_D_tilde} discusses how we move from ${D}_{\pi}(P^{*}, \widehat P)$ to $\widetilde{D}_{\pi}(P^{*}, \widehat P)$ in detail.

\subsection{Marginal importance weight training} \label{sec:main_method_miw}
A number of methods have been recently proposed to estimate the marginal importance weight $\omega$ \citep{breakingcurse2018,dualdice2019,gendice2020}. 
%Estimating the MIW $\omega$ is one important topic of recent OPE research, and several methods have been purposed \citep{breakingcurse2018,dualdice2019,gendice2020}.
These methods typically require solving a complex saddle-point optimization, 
casting doubts on their training stability especially when combined with policy learning on continuous-control offline MBRL problems.
In this section, we mimic the Bellman backup to derive a fixed-point-style method for estimating the MIW.

Denote the true MIW as $\omega^{*}(s,a)\triangleq\frac{d_{\pi,\gamma}^{P^{*}}(s,a)}{\dbgtrue(s,a)}$, we have $\dbgtrue(s,a) \cdot \omega^{*}(s,a) = d_{\pi,\gamma}^{P^{*}}(s,a)$.
Expanding the RHS, $\forall\, s',a'$,
\begin{equation}\label{eq:dr_target_derive}\textstyle
    \resizebox{0.92\textwidth}{!}{%
    $
    \begin{aligned}
        \dbgtrue(s',a') \omega^{*}(s',a')
        = \gamma \sum_{s,a}\pi(a' \given s') P^*(s'  \given  s,a) \omega^{*}(s,a) \dbgtrue(s,a)   +  (1-\gamma) \mu_0(s') \pi(a'  \given  s').
    \end{aligned}
$%
}
\end{equation}
The derivation is deferred to Appendix~\ref{sec:derive_dr_target}.
Therefore, a ``\textit{Bellman equation}'' for $\omega(s',a')$ is
\begin{equation*}
\resizebox{.90\textwidth}{!}{%
    $
    \begin{aligned}
    \omega(s',a') &= \gT \omega(s', a')\,,\\
    \gT \omega(s', a') & \triangleq \frac{\gamma \sum_{s,a}  \pi(a' \given s') P^*(s'  \given  s,a) \omega(s,a) \dbgtrue(s,a) +  (1-\gamma) \mu_0(s') \pi(a'  \given  s')}{\dbgtrue(s',a')} \,.
\end{aligned}
$%
}
\end{equation*}
Here $\gT$ can be viewed as the ``\textit{Bellman operator}'' for $\omega$. 
%In mimicking the Bellman backup, 
The update iterate defined by $\gT$ has the following convergence property, which is proved in Appendix~\ref{sec:proof_miw_iterate_coverge}.
\begin{proposition}\label{thm:miw_iterate_coverge}
On finite state-action space, if the current policy $\pi$ is close to the behavior policy $\pi_b$, then the iterate for $\omega$ defined by $\gT$ converges geometrically.
\end{proposition}
The assumption that $\pi$ is close to $\pi_b$ coincides with the regularization term in the policy-learning objective discussed in \Secref{sec:policy_learning}.

Unfortunately, the RHS of Eq.~\eqref{eq:dr_target_derive} is not estimable since we do not know the density values therein.
We therefore multiply both sides of Eq. \eqref{eq:dr_target_derive} by some test function and subsequently sum over $(s',a')$ on both sides to get a tractable objective that only requires samples from the offline dataset and the initial state-distribution $\mu_0$.
It is desired to choose a test function that can better distinguish the difference between the left-hand side (LHS) and the RHS of Eq. \eqref{eq:dr_target_derive}. 
A potential choice is the action-value function of the policy $\pi$, due to some primal-dual relationship between the stationary state-action density-(ratio) ($d_{\pi,\gamma}^{P^{*}}(s,a)$ or $\omega^{*}(s,a)$) and the action-value function \citep{tang2019doubly,kallus2020double,bestdice2020}.
A detailed discussion on the choice of the test function is provided in Appendix~\ref{sec:miw_diss}.
%Assume that the action-value function $Q_\pi^{\widehat P}$ is non-zero almost everywhere on $\sS \times \sA$.

Practically we use $Q_\pi^{\widehat P}$ as the test function.
Note that multiplying both sides of Eq.~\eqref{eq:dr_target_derive} by the same $Q_\pi^{\widehat P}$ does not undermine the convergence property, under mild conditions on Q. 
Mimicking the Bellman backup to sum over $(s',a')$ on both sides, with the notation $\dbgtrue(s,a,s') = \dbgtrue(s, a) P^*(s'\given s,a)$, 
\begin{equation} \label{eq:dr_target_final222}
    \resizebox{0.93\textwidth}{!}{%
    $   
    \begin{aligned}
        \overbrace{\mathbb{E}_{(s,a) \sim \dbgtrue} \left[\omega^{*}(s,a) \cdot Q_\pi^{\widehat P}(s,a) \right]}^{\ell_1(\omega^{*})}
        = \overbrace{\gamma \mathbb{E}_{\substack{(s,a,s') \sim \dbgtrue \\ a'\sim \pi(\cdot \given s')}} \left[ \omega^{*}(s,a) \cdot Q_\pi^{\widehat P}(s',a') \right] + (1-\gamma) \mathbb{E}_{\substack{s\sim \mu_0(\cdot ) \\ a \sim \pi(\cdot \given s)}}\left[ Q_\pi^{\widehat P}(s, a) \right]}^{\ell_2(\omega^{*})}\,.
        \end{aligned}
    $%
}
\end{equation}
%where $\pi$ and $Q_\pi^{\widehat P}$ are treated as fixed.
Thus for a given $\omega$, we can optimize $\omega$ by minimizing the difference between the RHS and the LHS of Eq.~\eqref{eq:dr_target_final222}.
For training stability, we use a target network $\omega^{\prime}(s,a)$ for the RHS, and the final objective for learning $\omega$ is 
\begin{align}\textstyle
   \br{\ell_1(\omega) - \ell_2\br{\omega^{\prime}}}^2\,, \label{eq:dr_target_final}
\end{align}
where the target network $\omega'(s,a)$ is soft-updated after each gradient step, motivated by $Q_{\vtheta_j'}$ and $\pi_{\vphi'}$.

Our proposed training method is closely related to VPM \citep{vpm2020}. By using the MIW $\omega(s,a)$ itself as the test function, VPM leverages the variational power iteration to train MIW iteratively. 
Instead, our approach uses the current action-value function as the test function, motivated by the primal-dual relationship between the MIW and the action-value function in off-policy evaluation. 
We compare the empirical performance of several alternative approaches in Section~\ref{sec:exp_ablation} and in Table~\ref{table:additional_ablation} of Appendix~\ref{sec:additional_tables}.

%\textbf{Training MIW.}
%The training objective for the MIW $\omega$, Eq. \eqref{eq:dr_target_final}, can be implemented by minimizing the difference between the RHS and LHS, with $\omega(s,a)$ on the RHS replaced by a target network $\omega'(s,a)$.
%Similar to $Q_{\vtheta_j'}$ and $\pi_{\vphi'}$, the target MIW $\omega'(s,a)$ is soft-updated after each gradient step.

\subsection{Practical implementation} \label{sec:main_method_imple}
In this section we briefly discuss some implementation details of our offline Alternating Model-Policy Learning (AMPL) method, whose main steps are in Algorithm~\ref{alg:simple}. 
Further details are in Appendix~\ref{sec:algo_details}.

\myparagraph{Dynamic model training.} 
%To train the model by Eq.~\eqref{eq:model_weighted_mle}, 
We adopt common practice in offline MBRL \citep[\eg,][]{mopo2020,combo2021} to use an ensemble of Gaussian probabilistic networks $\widehat P(\cdot \given s, a)$ and $\hat r(s, a)$ to parameterize the stochastic transition and reward.
We initialize the dynamic model by standard MLE training, and periodically update the model by minimizing Eq.~\eqref{eq:model_weighted_mle}.
%Except for the weighted loss, we follow the setups in \citet{mopo2020} wherever applicable. 
%The model is initialized by un-weighted MLE training. 

\myparagraph{Critic training.}
We use the conservative target in the offline RL literature \citep[\eg,][]{bcq2019,bear2019}:
\begin{equation} \textstyle \label{eq:q_tilde}
\resizebox{.93\textwidth}{!}{%
$
         \widetilde{Q}\br{s, a} \triangleq  r(s,a) + \gamma \E_{a'\sim \pi_{\vphi'}\br{\cdot \given s'}}[c \min_{j=1,2}Q_{\vtheta'_j}\br{s', a'} + (1-c) \max_{j=1,2}Q_{\vtheta'_j}\br{s', a'}]\,,
$
}%
\end{equation}
where we set $c=0.75$.
With mini-batch $\gB $ sampled from the augmented dataset $\gD$, both critic networks are trained as
\begin{equation}\label{eq:critic_target_main}\textstyle
    \forall\, j=1,2,\quad \arg\min_{\vtheta_j}\frac{1}{\abs{\gB}}\sum_{\br{s, a} \in \gB} \mathrm{Huber}( Q_{\vtheta_j}\br{s, a}, \widetilde{Q}(s, a))\,,
\end{equation}
where the Huber loss $\mathrm{Huber}(\cdot)$ is used in lieu of the classical MSE for training stability \citep{huber1964huber}.

% \textbf{Model rollouts.}
% % use \mathcal{D} to approximate discounted stationary distribution
% To maximize $J(\pi, \widehat P)$, we follow prior MBRL work \citep{mbpo2019, mopo2020} to use the augmented dataset $\gD$.
% While not required in Eq.~\eqref{eq:policy_objective}, in practice we find that adding synthetic rollouts to the regularizer helps the performance.
% We hypothesize that $\dmodel$ can strengthen the regularizer by mitigating the discrepancy between the offline dataset and the stationary state-action distribution of the current policy, while naturally serving as an augmentation to $\denv$ \citep{combo2021}.
% Therefore, in summary, we use $\gD$ in critic training, policy training, and regularizer construction.
% We follow \citet{combo2021} to use $f=0.5$ for our method.

% \textbf{Implicit policy.} 
% To better maximize the action-value function, we learn our policy in a rich function class, the class of implicit distributions.
% Motivated by the conditional GAN \citep{cgan2014}, the implicit policy fits a flexible state-conditional action distribution by transforming a given noise distribution $p_z(z)$ using a deterministic neural network.
% Specifically, given state $s$, stochastic action $a \sim \pi_\vphi(\cdot \given s) = \pi_\vphi(s, z)$ with $i.i.d.$ samples $z \sim p_z(z)$, where $\pi_\vphi$ is a deterministic network.

%\textbf{Regularizer construction.}
\myparagraph{Estimating $\widetilde{D}_{\pi}(P^{*}, \widehat P)$.}
% % why dual
% Because of our adoption of implicit policy, the density value $\pi(\cdot \given s)$ is unavailable.
% Hence the $\sqrt{\mathrm{KL}}$ term in Eq.~\eqref{eq:policy_objective} cannot be estimated in the primal form.
% % KL dual (write out form of dual), which is known to be unstable
% While the dual form of KL divergence is tempting, it however takes the form $\mathrm{KL}(p \,||\, q) = \sup_{T} \E_{p}[T] - \log(\E_q[e^T])$ \citep{kldual1983}, which involves exponential function and is known to be unstable.
% In the implementation, we use the Jensen-Shannon divergence (JSD) in lieu of the KL-dual. 
% Specifically, the $\sqrt{\mathrm{KL}}$ term in $\widetilde{\gL}(\pi, \widehat P)$ of Eq.~\eqref{eq:policy_objective} is changed into
% $
%     \mathrm{JSD}(P^*(s'\given s,a) \pi_b(a'\given s')\dbgtrue(s,a) \,||\, \widehat P(s' \given s,a)\pi(a'\given s') \dbgtrue(s) \pi(a\given s)).
% $
% Note that the JSD can be approximately minimized by the GAN structure \citep{gan2014}, with many stabilization techniques available in literature (see Appendix~\ref{sec:algo_details_gan}).
%Based on the structure of common offline RL datasets, using the notation of GAN, we implementation the regularizer
In $\widetilde{D}_{\pi}(P^{*}, \widehat P)$, using the notations of GAN,
we denote the sample from the left distribution of JSD in Eq.~\eqref{eq:approx_new_reg} by 
$\gB_{\mathrm{true}}$ (\ie,  ``true'' sample), 
and the sample from the right distribution of JSD by $\gB_{\mathrm{fake}}$ (\ie, ``fake'' sample).

% $\gB_{\mathrm{true}}$ actually consists of samples from $\denv$, 
% and $\gB_{\mathrm{fake}}$ are formed by first sampling $s \sim \gD$, followed by sampling actions using $a \sim \pi_\vphi(\cdot \given s)$.

% $\widetilde{D}_{\pi}(P^{*}, \widehat P)$ in Eq.~\eqref{eq:approx_new_reg} as follows.
The ``true'' sample $\gB_{\mathrm{true}}$ consists of samples from $\denv$.
The ``fake'' sample $\gB_{\mathrm{fake}}$ is formed by first sampling $s \sim \gD$, followed by 
$a \sim \pi_\vphi(\cdot \given s), s^{\prime} \sim \widehat P(\cdot \given s,a), a^{\prime} \sim \pi_\vphi(\cdot \given s^{\prime})$.
%sampling actions using $a \sim \pi_\vphi(\cdot \given s)$.
%Then for each $(s,a)$ a next state $s'$ is sampled using $\widehat P(\cdot \given s, a)$, followed by sampling an $a'$ for each $s'$ via $a' \sim \pi_\vphi(\cdot \given s')$.
Concretely, the ``fake'' sample $\gB_{\mathrm{fake}}$, generator loss $\gL_g(\vphi)$, and the discriminator loss $\gL_D(\vpsi)$ can be described as
\noindent\begin{minipage}{0.3\textwidth}
\begin{equation} \label{eq:fake_sample}
\gB_{\mathrm{fake}} \triangleq     
\begin{bmatrix}
    (s~, & a~) \\
    (s', & a')
    \end{bmatrix}\,,
\end{equation}
    \end{minipage}%
    \begin{minipage}{0.00\textwidth}\centering
\mbox{}
    \end{minipage}%
    \begin{minipage}{0.7\textwidth}
\begin{equation} \label{eq:generator_loss}
\textstyle
    \resizebox{0.85\textwidth}{!}{%
    $
      \gL_g(\vphi) \triangleq \frac{1}{\abs{\gB_{\mathrm{fake}}}} \sum_{(s,a) \in \gB_{\mathrm{fake}}}\sbr{\log\br{1-D_\vpsi\br{s,a}}}\,,
    $%
}
\end{equation}
\end{minipage} %\vskip2em
\begin{equation}\label{eq:dis_objective}
\textstyle
        \gL_D(\vpsi) \triangleq \frac{1}{\abs{\gB_{\mathrm{true}}}} \sum_{(s,a)\sim \gB_{\mathrm{true}}}\sbr{\log D_\vpsi(s,a)} + \frac{1}{\abs{\gB_{\mathrm{fake}}}} \sum_{(s,a)\sim \gB_{\mathrm{fake}}} \sbr{\log\br{1-D_\vpsi(s,a)}},
\end{equation}
where $\gL_g(\vphi)$ is the empirical policy-learning regularizer implied by $\widetilde D_{\pi}(P^{*}, \widehat P)$.

%is added into the policy learning objective as the regularizer.  

\myparagraph{Policy training.}
Since the constant multiplier in Eq.~\eqref{equ:new_lower_bound} is unknown, we treat it as a hyperparameter.
Adding the regularization term of policy training, our policy optimization objective is 
\begin{equation}\textstyle \label{eq:policy_target}
    \arg\min_{\vphi} -\lambda \cdot\frac{1}{\abs{\gB}} \sum_{s \in \gB, a \sim \pi_\vphi(\cdot \given s)}\sbr{\min_{j=1,2} Q_{\vtheta_j}(s, a)} + \gL_g(\vphi),
\end{equation}
where the regularization coefficient $\lambda \triangleq \lambda'/Q_{avg}$, with soft-updated $Q_{avg}$ similar to \citet{td3bc2021}.
We use $\lambda'=10$ across all datasets in our experiments (\Secref{sec:experiment}).

%Algorithm~\ref{alg:simple} summaries the main steps of our method.

% \begin{wrapfigure}{R}{0.65\textwidth}
% \begin{minipage}{0.65\textwidth}
% \vspace{-3em}
\begin{algorithm}[t]
\captionsetup{font=small}
\caption{Main steps of AMPL.}
\begin{algorithmic}
\label{alg:simple}
\STATE \textbf{Initialize:} Dynamic model $\widehat P$ and $\hat r$, policy $\pi_{\vphi}$, critics $Q_{\vtheta_1}$ and $Q_{\vtheta_2}$, discriminator  $D_\vpsi$, MIW $\omega$.
% \STATE Scale the rewards in the offline dataset.
\STATE Initialize $\widehat P$ and $\hat r$ via the MLE (Eq.~\eqref{eq:model_objective}).
\FOR{iteration $\in \{1, \ldots, {\tt total\_iterations}\}$}
\IF{iteration $\%$ {\tt model\_retrain\_period} == 0}
\STATE Estimate $\omega(s,a)$ via Eq.~\eqref{eq:dr_target_final}; train $\widehat P$ and $\hat r$ by weighted MLE (Eq.~\eqref{eq:model_weighted_mle}) with $\omega(s,a)$.
%\STATE Train MIW network $\omega$  using Eq.~\eqref{eq:dr_target_final}.
%\STATE Train $\widehat P$ by weighted MLE (Eq.~\eqref{eq:model_weighted_mle}) with $\omega(s,a)$.
\ENDIF
\STATE Rollout synthetic data with $\pi_\vphi$, $\widehat P$ and $\hat r$, and add the data to $\dmodel$.
\STATE Sample mini-batch $\gB \sim \gD = f\denv + (1-f) \dmodel$. 
\STATE Optimize $Q_{\vtheta_1}$, $Q_{\vtheta_2}$ via Eqs.~\eqref{eq:q_tilde} -- \eqref{eq:critic_target_main}.
%Get critic target by Eq.~\eqref{eq:q_tilde} and train critics by Eq.~\eqref{eq:critic_target_main}.
%\STATE Construct $\gB_{\mathrm{fake}}$ by Eq.~\eqref{eq:fake_sample} and sample $\gB_{\mathrm{true}}\sim \denv$.
%\STATE Calculate generator loss $\gL_g(\vphi)$ using Eq.~\eqref{eq:generator_loss}.
\STATE Train $D_\vpsi$ to maximize Eq.~\eqref{eq:dis_objective}.
\STATE Optimize $\pi_{\vphi}$ by Eq.~\eqref{eq:policy_target}.
\ENDFOR
\end{algorithmic}
\end{algorithm}
% \vspace{-2.5em}
% \end{minipage}
% \end{wrapfigure}

\section{Experiments} \label{sec:experiment}

In this section, we first evaluate our AMPL on continuous-control offline-RL datasets in Section~\ref{sec:exp_main}.
Then we conduct ablation study on Section~\ref{sec:exp_ablation} to understand the efficacy of some proposed designs.
% \subsection{Off-policy evaluation results}\label{sec:exp_ope}

\subsection{Continuous-control results}\label{sec:exp_main}
% As discussed in \Secref{sec:main_method_imple}, for training stability, in the implementation the $\sqrt{\mathrm{KL}}$ term in the policy-training regularizer $\widetilde{\gL}(\pi, \widehat P)$ of Eq.~\eqref{eq:policy_objective} is substituted by the JSD, which can be approximately minimized by the GAN structure.
Our experiments are conducted on a diverse set of datasets in the D4RL benchmark \citep{fu2021d4rl}, ranging across the Gym-Mojoco, Maze2D, and Adroit domains therein.
We use the latest version of the datasets, \ie, ``v2'' version for the Gym-Mojoco domain and ``v1'' version for the Maze2D and Adroit domains.
Details of our dataset choice are discussed in Appendix~\ref{sec:rl_exp_details}.

We compare our AMPL with two DICE methods --- AlgaeDICE \citep{algaedice2019} and OptiDICE \citep{optidice2021} --- 
that directly utilize MIW to improve policy learning.
We further consider three state-of-the-art (SOTA) offline model-free RL algorithms: CQL \citep{cql2020}, FisherBRC \citep{fisherbrc2021}, and TD3+BC \citep{td3bc2021}; and three SOTA offline MBRL
methods: MOPO \citep{mopo2020}, COMBO \citep{combo2021}, and WMOPO \citep{wmopo2021}.
Experimental details and hyperparameter settings are discussed in Appendix~\ref{sec:rl_exp_details}.
We run baseline methods using the official implementation under the recommended hyperparameters, except for AlgaeDICE, for which we use the offline version provided by \citet{fu2021d4rl}.
Table~\ref{table:main} shows the mean and standard deviation of the results of AMPL and baselines over five random seeds.

As shown in Table~\ref{table:main}, our AMPL performs comparably well and is relatively stable across the sixteen tested datasets.
In particular, AMPL generally performs better than the baselines on the Maze2D and Adroit datasets.
The Maze2D and Adroit datasets are considered as more challenging than the MuJoCo datasets \citep{fu2021d4rl},
since the Maze2D datasets are collected by non-Markovian policies,
and the amount of data in the high-dimensional Adroit datasets is limited. 

On the Maze2D datasets, the behavioral-cloning (BC) style algorithms, such as FisherBRC, are likely to fail, 
since these methods use Markovian policy to approximate the non-Markovian behavior policy. 
Meanwhile, OptiDICE performs relatively well on the Maze2D datasets, mainly because it uses a Gaussian-mixture policy for BC to alleviate this approximation difficulty.
This advanced BC policy still offers little help on the higher-dimensional Adroit datasets, due to the challenge of learning the behavior policy.
% In high dimensional Adroit tasks, policies learned by FisherBRC do not generalize well amount to the limited offline data.
Similarly, TD3+BC performs well on the relatively lower-dimensional Maze2D datasets,
but does not work well on the Adroit domain, 
since the action-space MSE-regularizer in TD3+BC may be insufficient for the high-dimensional tasks.
%the Adroit datasets are featured by its narrow data-distribution and high dimension.
% On these two domains, methods based on behavioral cloning (BC), such as FisherBRC, are likely to fail.
% Main difficulties of BC on these two domains are the error of using a Markovian policy to approximate the non-Markovian behavior policy for the Maze2D datasets, and the challenges of learning and constraining towards the behavior policy in high-dimensional space with limited data for the Adroit datasets.
% OptiDICE performs well on the Maze2D tasks, though, where the stated difficulty is partially alleviated by its adoption for BC of mixture of Gaussian policy with several mixtures.
% This advanced BC design is still of little help in the higher dimensional Adroit datasets.
% Similarly, TD3+BC's action-space MSE-regularizer performs well on relatively lower-dimensional Maze2D datasets, it is again insufficient for the high-dimensional Adroit domain.
% why wmopo does not perform well
% WMOPO, which utilizes full-trajectory model-rollout to estimate the weights for dynamic model learning,
% does not perform well on these two challenging domains, since the accumulated errors can not be bounded in long rollouts. 
WMOPO also shows a performance drop in these two challenging task domains.
The environmental dynamics on these two domains are complex, and thus may not be accurately learned to support WMOPO's full-trajectory model rollouts, which are used to estimate its MIW.
By contrast, our AMPL does not require BC or long model-rollouts, leading to stable performance across task domains. 
Moreover, AMPL shows generally better performance than MOPO and COMBO, which use fixed pretrained dynamic models without considering the mismatched model objectives.

% simple advantage over DICEs
Compared with the DICE-based methods AlgaeDICE and OptiDICE, 
our AMPL shows generally better performance.
This may indicate that maximizing a lower bound of the true expected return can be a more effective framework than explicit stationary-distribution correction, since maximizing the lower bound is more directly related to RL's goal of maximizing the policy performance.
% offline RL algorithms can enjoy more benefits of utilizing MIW for model training, instead of directly applying MIW for policy optimization.
%the IMW can be potentially more helpful for 
%model training, instead of the explicit usage for policy optimization.
%The comparision between our method with the two DICE-methods may indicate that 

\begin{table}[tb] 
\captionsetup{font=small}
\caption{
\footnotesize Normalized returns for experiments on the D4RL tasks.
Mean and standard deviation across five random seeds are reported.
High average score and low average rank are desirable.
We bold the best result over all methods and underline the best of the model-based methods if different.
Here, ``hcheetah'' denotes ``halfcheetah,'' ``med'' denotes ``medium,'' ``rep'' denotes ``replay,'' and ``exp'' denotes ``expert.''} 
\label{table:main} 
\vspace{.5em}

\centering 
\def\arraystretch{1.2}
\resizebox{1.\textwidth}{!}{
\begin{tabular}{lccccccccc}
\toprule
                 Task Name &                        AlgaeDICE &                          OptiDICE &                               CQL &                        FisherBRC &                            TD3+BC &                             MOPO &                             COMBO &                        WMOPO &                              AMPL \\
\midrule
              maze2d-large &   -2.2 $\pm$ {\footnotesize 0.6} &  101.7 $\pm$ {\footnotesize 50.7} &     1.5 $\pm$ {\footnotesize 6.4} &    0.9 $\pm$ {\footnotesize 6.4} &  107.1 $\pm$ {\footnotesize 45.9} &   -0.5 $\pm$ {\footnotesize 2.6} &  138.5 $\pm$ {\footnotesize 82.2} &    1.8 $\pm$ {\footnotesize 8.6} &  \textbf{180.0} $\pm$ {\footnotesize 39.3} \\
             maze2d-med &   2.5 $\pm$ {\footnotesize 14.2} &  \textbf{119.4} $\pm$ {\footnotesize 52.3} &     6.3 $\pm$ {\footnotesize 9.1} &  16.1 $\pm$ {\footnotesize 30.3} &   61.4 $\pm$ {\footnotesize 45.5} &  12.5 $\pm$ {\footnotesize 18.3} &  103.9 $\pm$ {\footnotesize 42.1} &  12.7 $\pm$ {\footnotesize 38.3} &  \underline{107.2} $\pm$ {\footnotesize 45.0} \\
              maze2d-umaze &  -15.3 $\pm$ {\footnotesize 0.8} &  \textbf{114.0} $\pm$ {\footnotesize 39.7} &    37.5 $\pm$ {\footnotesize 7.2} &   3.6 $\pm$ {\footnotesize 16.4} &   38.6 $\pm$ {\footnotesize 14.4} &  -15.4 $\pm$ {\footnotesize 1.9} &  \underline{112.1} $\pm$ {\footnotesize 56.8} &  -11.5 $\pm$ {\footnotesize 3.1} &    55.8 $\pm$ {\footnotesize 3.9} \\
        \midrule
        hcheetah-med &   -0.9 $\pm$ {\footnotesize 0.7} &    42.0 $\pm$ {\footnotesize 3.5} &    48.6 $\pm$ {\footnotesize 0.2} &   47.8 $\pm$ {\footnotesize 0.3} &    48.0 $\pm$ {\footnotesize 0.3} &   69.2 $\pm$ {\footnotesize 3.4} &    \textbf{73.0} $\pm$ {\footnotesize 3.6} &   72.0 $\pm$ {\footnotesize 4.7} &    51.7 $\pm$ {\footnotesize 0.4} \\
           walker2d-med &    1.0 $\pm$ {\footnotesize 2.1} &   55.7 $\pm$ {\footnotesize 15.7} &    81.8 $\pm$ {\footnotesize 1.7} &   80.7 $\pm$ {\footnotesize 2.1} &    83.0 $\pm$ {\footnotesize 1.4} &   -0.1 $\pm$ {\footnotesize 0.0} &     0.5 $\pm$ {\footnotesize 0.9} &  64.8 $\pm$ {\footnotesize 30.0} &    \textbf{83.1} $\pm$ {\footnotesize 1.8} \\
             hopper-med &    0.9 $\pm$ {\footnotesize 0.2} &    57.5 $\pm$ {\footnotesize 7.5} &    68.0 $\pm$ {\footnotesize 6.0} &   94.2 $\pm$ {\footnotesize 4.5} &    59.1 $\pm$ {\footnotesize 4.5} &  44.8 $\pm$ {\footnotesize 41.5} &   22.6 $\pm$ {\footnotesize 38.1} &   \textbf{99.7} $\pm$ {\footnotesize 2.9} &    58.9 $\pm$ {\footnotesize 7.9} \\
 hcheetah-med-rep &   -3.4 $\pm$ {\footnotesize 2.3} &    40.5 $\pm$ {\footnotesize 3.3} &   38.7 $\pm$ {\footnotesize 18.6} &  34.9 $\pm$ {\footnotesize 18.2} &    44.5 $\pm$ {\footnotesize 0.6} &   62.7 $\pm$ {\footnotesize 7.5} &    66.0 $\pm$ {\footnotesize 1.8} &   \textbf{66.4} $\pm$ {\footnotesize 4.9} &    44.6 $\pm$ {\footnotesize 0.7} \\
    walker2d-med-rep &    0.5 $\pm$ {\footnotesize 0.6} &   37.3 $\pm$ {\footnotesize 21.8} &    80.8 $\pm$ {\footnotesize 4.0} &   \textbf{84.6} $\pm$ {\footnotesize 6.7} &    74.7 $\pm$ {\footnotesize 8.1} &  53.8 $\pm$ {\footnotesize 34.6} &   60.1 $\pm$ {\footnotesize 18.3} &  71.9 $\pm$ {\footnotesize 15.0} &    \underline{81.5} $\pm$ {\footnotesize 3.0} \\
      hopper-med-rep &    1.5 $\pm$ {\footnotesize 0.7} &   28.1 $\pm$ {\footnotesize 12.4} &    \textbf{96.0} $\pm$ {\footnotesize 4.7} &   94.4 $\pm$ {\footnotesize 1.9} &   58.9 $\pm$ {\footnotesize 19.7} &  84.8 $\pm$ {\footnotesize 30.0} &   53.6 $\pm$ {\footnotesize 29.5} &   \underline{93.8} $\pm$ {\footnotesize 8.5} &    91.1 $\pm$ {\footnotesize 9.5} \\
 hcheetah-med-exp &   -1.8 $\pm$ {\footnotesize 2.9} &   66.7 $\pm$ {\footnotesize 25.8} &   53.8 $\pm$ {\footnotesize 13.6} &   \textbf{94.5} $\pm$ {\footnotesize 0.6} &    91.4 $\pm$ {\footnotesize 5.2} &  64.0 $\pm$ {\footnotesize 20.8} &   52.5 $\pm$ {\footnotesize 32.0} &  68.8 $\pm$ {\footnotesize 37.5} &    \underline{90.2} $\pm$ {\footnotesize 2.0} \\
    walker2d-med-exp &   -0.2 $\pm$ {\footnotesize 0.1} &   79.4 $\pm$ {\footnotesize 16.4} &   110.3 $\pm$ {\footnotesize 0.6} &  109.3 $\pm$ {\footnotesize 0.1} &   \textbf{110.4} $\pm$ {\footnotesize 0.6} &   -0.2 $\pm$ {\footnotesize 0.0} &     1.1 $\pm$ {\footnotesize 0.6} &  98.8 $\pm$ {\footnotesize 16.2} &   \underline{107.7} $\pm$ {\footnotesize 2.1} \\
      hopper-med-exp &    2.2 $\pm$ {\footnotesize 1.9} &    52.6 $\pm$ {\footnotesize 9.3} &   79.3 $\pm$ {\footnotesize 20.5} &  \textbf{110.0} $\pm$ {\footnotesize 4.0} &    98.7 $\pm$ {\footnotesize 9.6} &   15.6 $\pm$ {\footnotesize 7.4} &   63.7 $\pm$ {\footnotesize 51.9} &  62.6 $\pm$ {\footnotesize 26.4} &   \underline{76.0} $\pm$ {\footnotesize 15.7} \\
                 \midrule
                 pen-human &   -3.0 $\pm$ {\footnotesize 0.8} &    -0.9 $\pm$ {\footnotesize 3.1} &    -1.5 $\pm$ {\footnotesize 2.8} &   -0.1 $\pm$ {\footnotesize 3.0} &    1.6 $\pm$ {\footnotesize 10.2} &   -1.6 $\pm$ {\footnotesize 2.6} &     5.9 $\pm$ {\footnotesize 8.7} &   -2.1 $\pm$ {\footnotesize 1.4} &   \textbf{20.6} $\pm$ {\footnotesize 10.7} \\
                pen-cloned &   -2.6 $\pm$ {\footnotesize 1.3} &    -0.8 $\pm$ {\footnotesize 3.1} &   39.5 $\pm$ {\footnotesize 23.6} &   -1.2 $\pm$ {\footnotesize 3.8} &     6.3 $\pm$ {\footnotesize 4.6} &   5.2 $\pm$ {\footnotesize 10.2} &   23.2 $\pm$ {\footnotesize 19.8} &   -3.0 $\pm$ {\footnotesize 2.4} &   \textbf{57.4} $\pm$ {\footnotesize 19.2} \\
                pen-exp &   -1.3 $\pm$ {\footnotesize 2.4} &     0.4 $\pm$ {\footnotesize 7.3} &  119.3 $\pm$ {\footnotesize 15.8} &    2.2 $\pm$ {\footnotesize 1.1} &  104.2 $\pm$ {\footnotesize 40.6} &  35.4 $\pm$ {\footnotesize 20.9} &   57.1 $\pm$ {\footnotesize 42.6} &  10.4 $\pm$ {\footnotesize 20.9} &   \textbf{138.3} $\pm$ {\footnotesize 6.6} \\
               door-exp &    0.1 $\pm$ {\footnotesize 0.0} &   86.1 $\pm$ {\footnotesize 23.2} &   94.0 $\pm$ {\footnotesize 15.8} &  22.9 $\pm$ {\footnotesize 28.8} &    -0.3 $\pm$ {\footnotesize 0.0} &   -0.0 $\pm$ {\footnotesize 0.1} &    -0.2 $\pm$ {\footnotesize 0.2} &   -0.1 $\pm$ {\footnotesize 0.1} &    \textbf{96.3} $\pm$ {\footnotesize 9.8} \\
               \midrule
             Average Score &                             -1.4 &                                55 &                              59.6 &                             49.7 &                              61.7 &                             26.9 &                              52.1 &                             44.2 &                              \textbf{83.8} \\
              Average Rank &                              8.5 &                               5.5 &                               4.1 &                              4.3 &                               3.9 &                              6.4 &                               4.8 &                                5.0 &                               \textbf{2.6} \\
\bottomrule
\end{tabular}
}
\end{table}

\begin{figure}[ht]
     \begin{subfigure}[b]{0.334\textwidth}
         \centering
         \includegraphics[width=\textwidth]{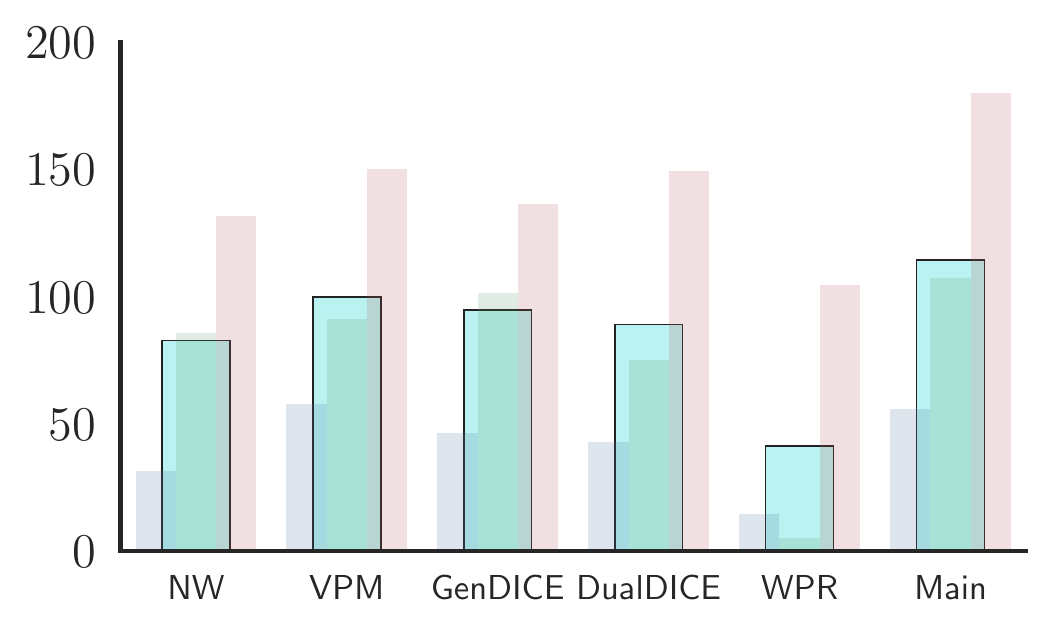}
         \captionsetup{font=small}
         %\vspace{-1.5em}
         \caption{\footnotesize Maze2D}
         \label{fig:abla_maze}
     \end{subfigure}
     \hspace{-.3em}%
     \begin{subfigure}[b]{0.334\textwidth}
         \centering
         \includegraphics[width=\textwidth]{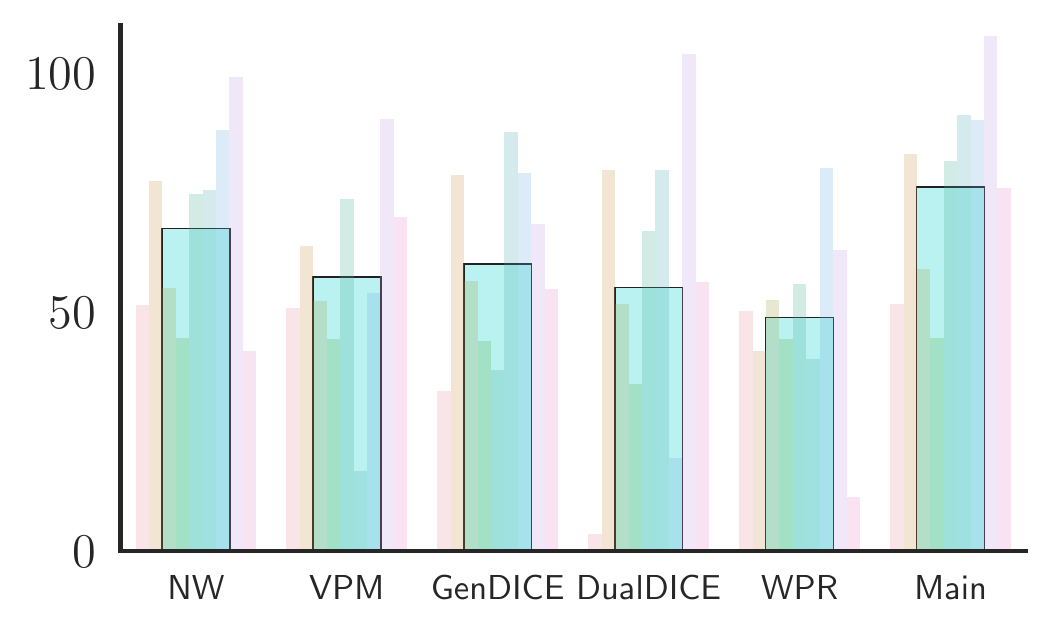}
         \captionsetup{font=small}
         %\vspace{-1.5em}
         \caption{\footnotesize MuJoCo}
         \label{fig:abla_mujoco}
     \end{subfigure}
     \hspace{-.3em}%
     \begin{subfigure}[b]{0.334\textwidth}
         \centering
         \includegraphics[width=\textwidth]{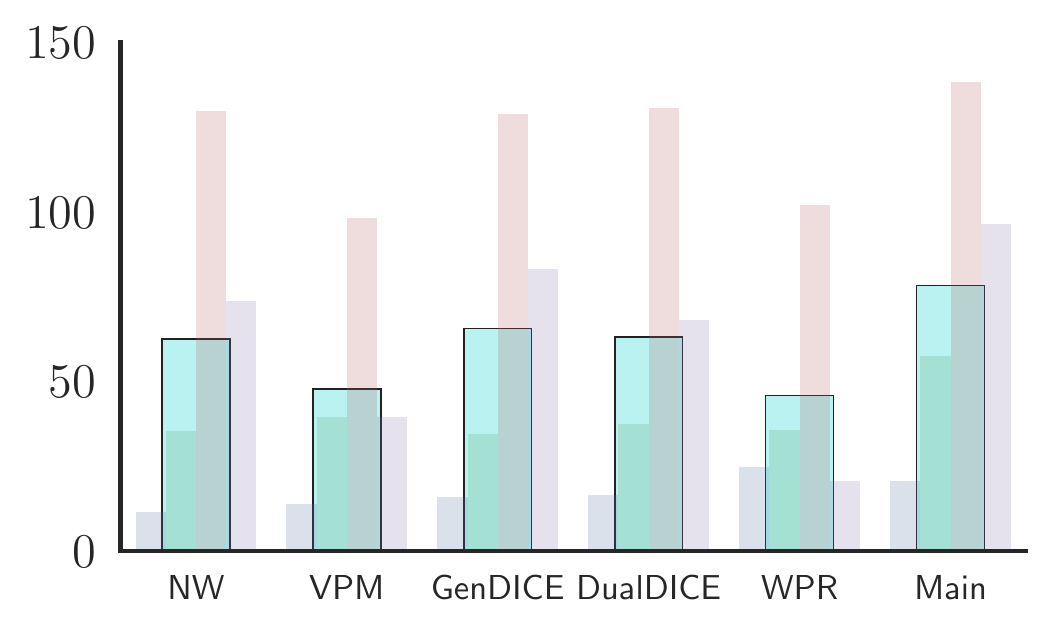}
         \captionsetup{font=small}
         %\vspace{-1.5em}
         \caption{\footnotesize Adroit}
         \label{fig:abla_adroit}
     \end{subfigure}
     %\vspace{-1.8em}
     \captionsetup{font=small}
    \caption{ 
        \small
        Scores of each method in the ablation study (\Secref{sec:exp_ablation}) on each domain of datasets.
        The faded bars show the scores on each datasets within the stated domain, averaged over five seeds, where each color corresponds to a dataset.
        The highlighted bar shows the average score on the entire domain.
        Label ``Main'' refers to the results of our main method in Section~\ref{sec:exp_main}.
        Detailed numbers are on Table~\ref{table:ablation_sec5}.
        % ``GDice'' denotes ``GenDICE'' and ``DDice'' denotes ''DualDICE.''
        }
        \label{fig:ablation}
        %\vspace{-1em}
\end{figure}

\subsection{Ablation study} \label{sec:exp_ablation}
% The ablation study is devoted to answering the following questions:

% \textbf{(a)} Does the MIW-weighted model (re)training indeed help the performance?

% \textbf{(b)} How does the algorithm perform if we change our MIW estimation to other approaches?

% \textbf{(c)} What is the performance of the weighted regularizer for policy learning implied by Theorem~\ref{thm:eval_error_upper_bound}?

% % Detailed experimental setup and full results can be found in Appendix \red{xxx}.
% Unless explicitly stated, the experimental setups for all algorithmic variants follow Section~\ref{sec:exp_main}.

\textbf{(a):} 
{\it Does the MIW-weighted model (re)training help the performance?} 

To verify the effectiveness of our MIW-weighted model (re)training scheme, we compare our AMPL with its variant of training the model only at the beginning using MLE, \ie, No Weights (dubbed as NW).
As shown in Figure~\ref{fig:ablation}, the performance of the NW variant is generally worse than the main method (dubbed as Main) on all three domains.
The performance difference is especially significant on the Maze2D domain.
In this domain, the policy is required to ``stitch together collected subtrajectories to find the shortest path to the evaluation goal'' \citep{fu2021d4rl}. As a result, the state-action distribution induced by a good policy should be different from the distribution of the offline-data.
Table~\ref{table:ablation_sec5} in Appendix \ref{sec:additional_tables} confirms the benefit of the MIW-weighted model (re)training.
Indeed, on the majority of datasets, the NW variant not only shows worse performance, but also has a larger standard deviation relative to the mean score.

\textbf{(b):} 
{\it How does the algorithm perform if we change our MIW estimation method to other approaches?}

\begin{figure}[t]
\centering
\includegraphics[width=.5\textwidth]{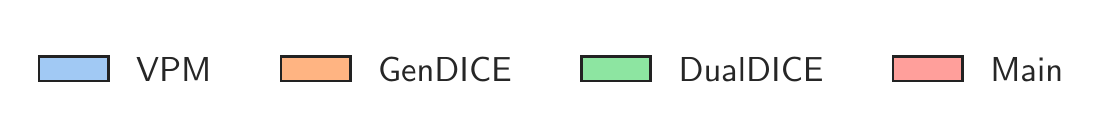}
\\
%\vspace{-2.2mm}
\begin{subfigure}[b]{0.49\textwidth}
         \centering
         \includegraphics[width=\textwidth]{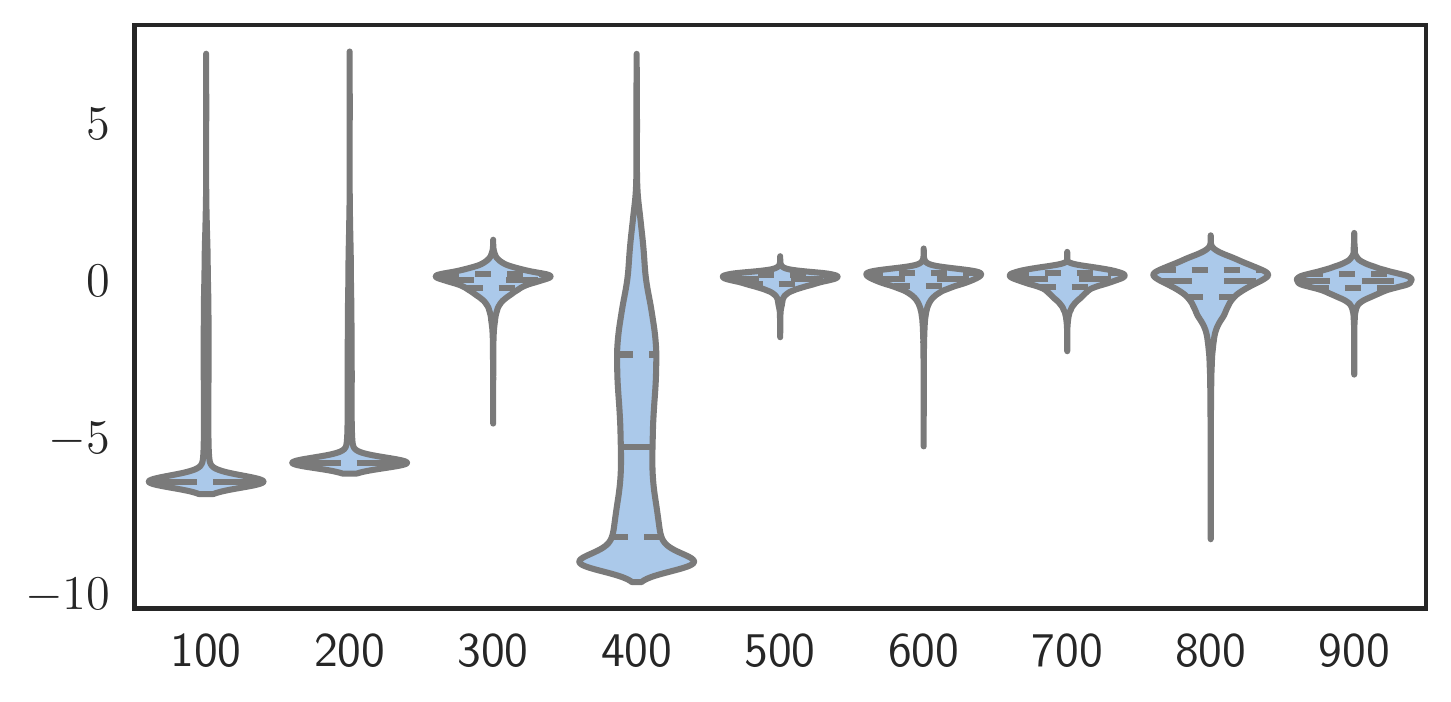}
         \captionsetup{font=small}
         %\vspace{-1.5em}
         \caption{\small VPM}
         \label{fig:abla_distplot_vpm}
     \end{subfigure}
     \hfill
     \begin{subfigure}[b]{0.49\textwidth}
         \centering
         \includegraphics[width=\textwidth]{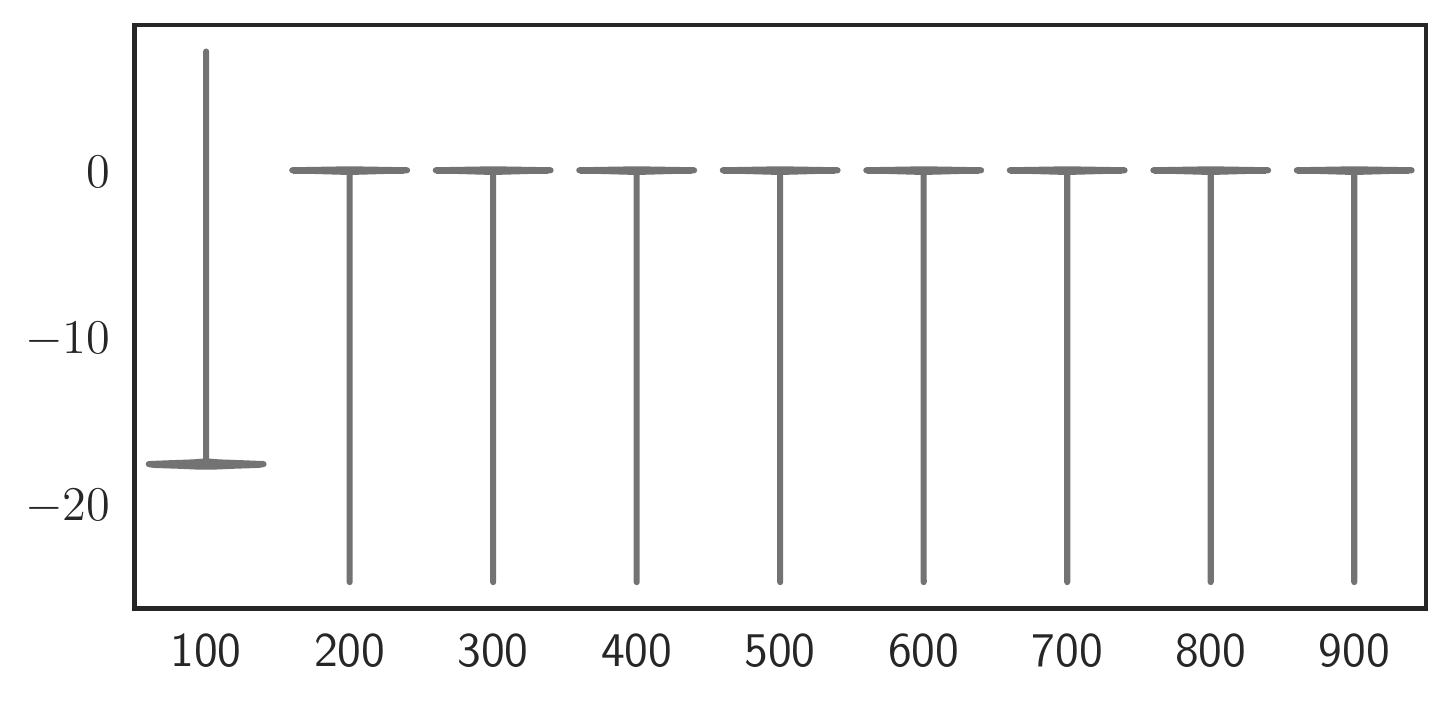}
         \captionsetup{font=small}
         %\vspace{-1.5em}
         \caption{\small GenDICE}
         \label{fig:abla_distplot_gendice}
     \end{subfigure}
     \\
     \begin{subfigure}[b]{0.49\textwidth}
         \centering
         \includegraphics[width=\textwidth]{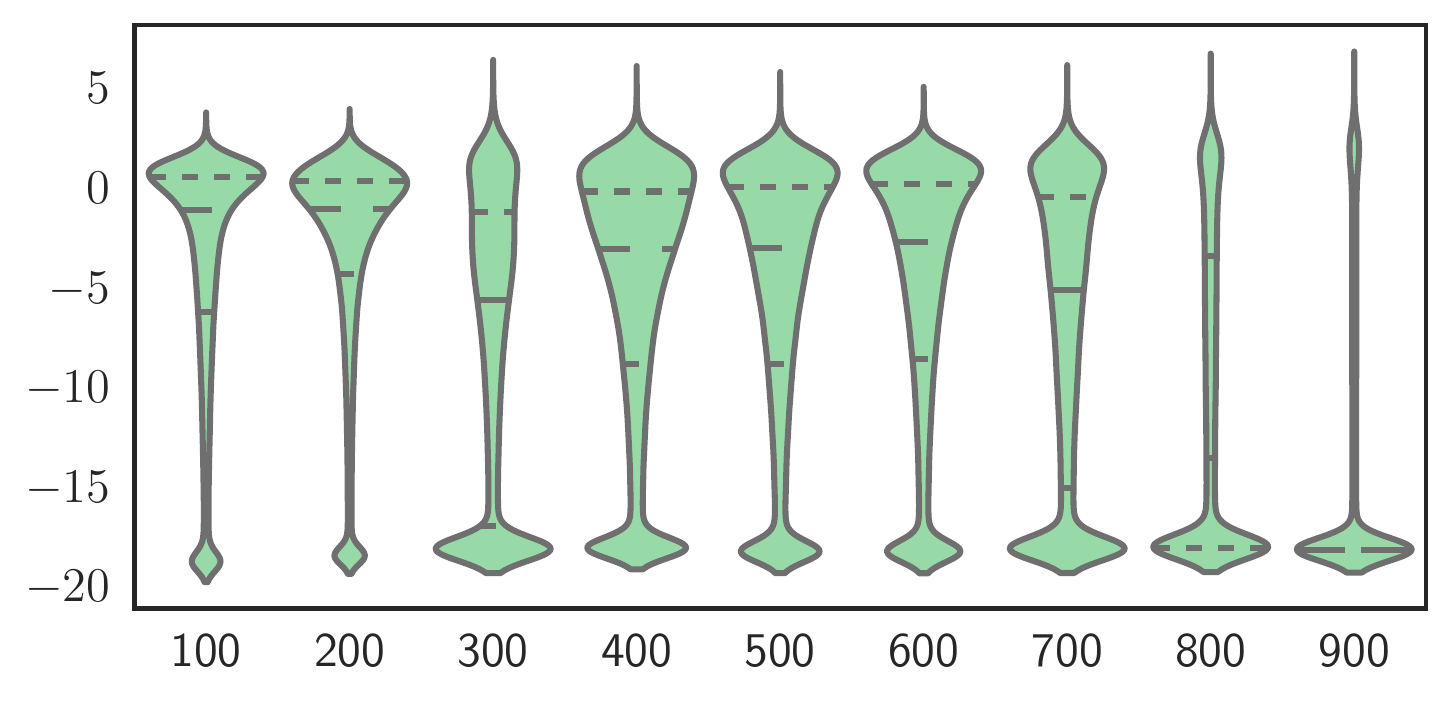}
         \captionsetup{font=small}
         %\vspace{-1.5em}
         \caption{\small DualDICE}
         \label{fig:abla_distplot_dualdice}
     \end{subfigure}
     \hfill
     \begin{subfigure}[b]{0.49\textwidth}
         \centering
         \includegraphics[width=\textwidth]{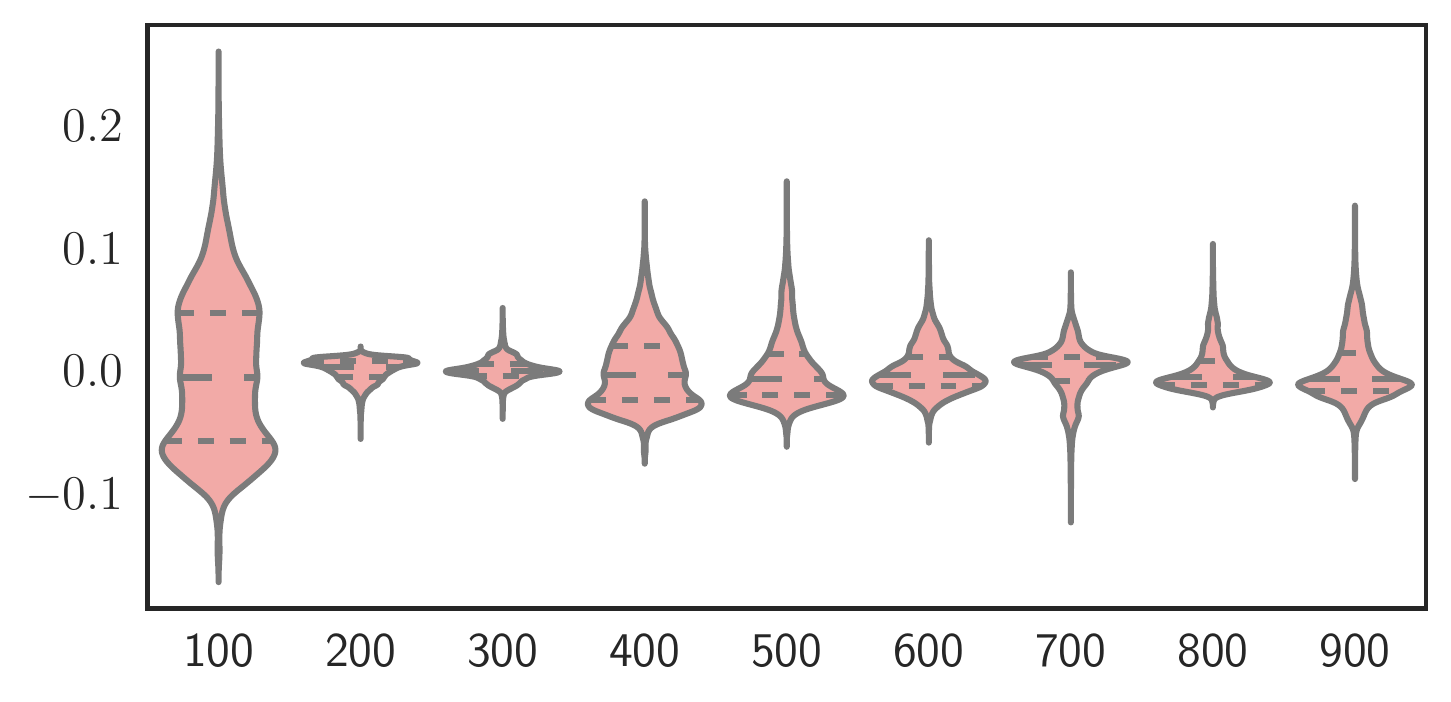}
         \captionsetup{font=small}
         %\vspace{-1.5em}
         \caption{\small Our method (``Main'')}
         \label{fig:abla_distplot_main}
     \end{subfigure}
    % \vspace{-.5em}
    \caption{ 
        \small
        Distribution plots of $\log\br{\text{MIW}}$ of the entire dataset by our method (``Main'') and the three alternatives in \Secref{sec:exp_ablation} \textbf{(b)}, on the ``walker2d-medium-replay'' dataset during the training process.
        The dash lines represent the quartiles of the distribution.
        Recall that the MIW $\omega$ is retrained every 100 epochs.
        The $x$-axis represents the number of epochs $\cbr{100, 200, \ldots, 900}$, and the $y$-axis the $\log\br{\text{MIW}}$ values.
        For clarity, we use the normalized MIWs whose mean on the entire dataset is one.
        We use random seed two for this plot.
        Other random seeds have similar patterns.
        We note the different scales in the $y$-axis of these four plots.
        }
        \label{fig:ablation_w}
        \vspace{-1.5em}
\end{figure}

% GenDICE weight degenerate, almost all 1
We compare our AMPL with its three variants where the MIW is instead estimated by the Variational Power Method \citep{vpm2020}, GENeralized stationary DIstribution Correction Estimation \citep{gendice2020}, and Dual stationary DIstribution Correction Estimation \citep{dualdice2019}.
Three variants are simply dubbed as VPM, GenDICE, and DualDICE.
For numerical stability, the estimated MIW from these three methods is clipped into $(10^{-8}, 500)$.
Implementation details are provided in Appendix~\ref{sec:dice_details}.

% other methods not stable
As shown in Figure~\ref{fig:ablation} and Table~\ref{table:ablation_sec5}, the variants with these three alternative MIW estimation methods generally perform worse than our approach.
A plausible explanation is that these methods can be unstable to provide good MIW estimates for the model training.
%or even worse, degenerated.
Figure~\ref{fig:ablation_w} shows the distribution plots of $\log\br{\text{MIW}}$ of the entire dataset estimated by our method and by the variants using the three alternative MIW estimation methods, on the ``walker2d-medium-replay'' dataset. 
We train the MIW every $100$ epochs and ignore the MIW initialization from the plots.
Notice that the MIWs obtained by these four methods have very different scales, which indicates the instability of the alternative methods compared with ours.

As shown on Figure~\ref{fig:abla_distplot_gendice}, the MIWs estimated by the GenDICE variant degenerate, in a sense that most of the MIWs are $\approx 0$ at the begining and are $1$ later on, with long tails on the MIW distributions.
%The MIW estimated by GenDICE degenerates as almost $0$ at first and then as almost $1$.
% The MIW estimated by GenDICE degenerates, at first being almost all $0$, then almost all $1$.
%at first being almost zeros $0$, then ones.
This may explain the relatively bad performance of GenDICE.
Compared with GenDICE, DualDICE provides more diverse MIW estimates, which leads to its relatively better performance. 
Unfortunately, Figure~\ref{fig:abla_distplot_dualdice} shows that, for the DualDICE variant, the distribution of MIWs gradually concentrates on very small and very large values, which indicates the degeneration of the MIWs.
The VPM variant performs the best among these three alternatives, mainly because it provides relatively better MIW estimates in the second half of the training process.
However,
Figure~\ref{fig:abla_distplot_vpm} shows that the MIWs from the VPM variant are poorly-distributed at the beginning, though relatively better later on.
%even better results, which may be related to its well-distributed MIW estimates in the second half of the policy training, though not so in the beginning.
Our method provides well-behaved MIW estimates in the whole training process, which may explain our better result. 
As shown by Figure~\ref{fig:abla_distplot_main}, the MIWs from our method are well-shaped and concentrate around the mean 1 \footnote{0 in the $\log\br{\text{MIW}}$ plots.} over the entire training process.
In general, the well-shaping of the estimated MIW is important, since the Cramer-Rao lower bound of the mean-square-error for OPE is related to the square of the density ratio \citep{kallus2020double,jiang16doubly}.

Besides, on the Maze2D domain, these three variants using the alternative MIW estimation methods perform generally better than the NW variant discussed in Question \textbf{(a)},
%methods all generally perform better than the NW variant in Part \textbf{(a)}.
aligning with the benefit of the MIW-weighted model (re)training scheme.
%discussed in
%Question \textbf{(a)}.
Generally speaking, incorporating the MIW can help model training in offline MBRL.
%may bring benefits for dynamic model learning.
%this scheme can help transfer advance in the MIW estimation to improvement in the offline MBRL.

\textbf{(c):} 
{\it What is the performance of a weighted regularizer for policy learning implied by Theorem~\ref{thm:eval_error_upper_bound}?}

We compare our AMPL with its variant where the  policy regularizer 
is weighted by the MIW $\omega(s,a)$, as suggested by Theorem~\ref{thm:eval_error_upper_bound}.
%in policy training is defined by the weighted regularizer in Theorem~\ref{thm:eval_error_upper_bound}.
%{For numerical stability, we follow Section~\ref{sec:policy_learning} to remove the $\sqrt{\cdot}$ and change the KL(-dual) divergence into the JSD.}
%his regularizer can be intuitively viewed as an MIW-weighted regularization of the action choice at the next state, which affects the policy evaluation target $J(\pi, P^*)$ through the $\E_{\pi}[Q(s',a')]$ term in the Bellman backup of $Q(s,a)$.
%The weighting scheme emphasizes $(s,a)$ pairs that are more frequently visited by the learned policy $\pi$.
This Weighted Policy Regularizer variant is dubbed as WPR. 
As shown in both Figure~\ref{fig:ablation} and Table~\ref{table:ablation_sec5}, the WPR variant underperforms our main method.
This is because when we estimate the regularization term in WPR,
we incorporate weights into the minimax optimization of the policy and the discriminator, 
which may bring additional instability. % for estimating the regularization term.

Table~\ref{table:additional_ablation} of Appendix~\ref{sec:additional_tables} provides additional ablation study on the performance of several alternatives.

% \vspace{-.5em}
%This is because the weights are directly incorporated in the regularizer estimation via adversarial training, 
%which may bring additional instability during policy learning.

\section{Related work} \label{sec:related_work}
\myparagraph{Offline MBRL.}
Most of the existing offline MBRL works focus on policy learning under a given model trained by MLE. 
These works typically constrain the learned policy to avoid visiting regions where the discrepancy between the true and the learned dynamic is large, thus reducing the policy evaluation error of using the learned dynamic \citep{mopo2020,morel2020,lompo2021,revisitdesignchoice2021}.
Besides, some recent works \citep{jointmatching2022,sdmgan2022} also adopt a GAN-style stochastic policy for its flexibility.
Rather than using a fixed MLE-trained model, we derive an objective that trains both the policy and the dynamic model toward maximizing a lower bound of true expected return 
(simultaneously minimizing the policy evaluation error $|J(\pi, P^{*}) - J(\pi, \widehat P)|$).
Several recent works also propose to enhance model training, such as training a reverse dynamic model to encourage conservative model-based imaginations \citep{romi2021}, 
learning a balanced state-action representation for the model to mitigate the distributional discrepancy \citep{repbmdp2018,repbsde2021}, and using advanced dynamic-model architecture of the GPT-Transformer \citep{gpt2018, fan2020bayesian, zhang2021bayesian, zhang2022allsh} to achieve accurate predictions \citep{trajectorytransformer2021}.
These methods are orthogonal to ours and may further improve the performance.

\myparagraph{Objective mismatch in MBRL.}
Our AMPL is related to prior works on online and offline MBRL that mitigate the mismatched model objectives.
In online (off-policy) MBRL, \citet{objectivemismatch2020} identify the mismatched objectives between the MLE model-training and the model's usage of improving the control performance.
This paper, however, only proposes an impractical weighted learning loss that requires the optimal trajectory. % knowing in advance
\citet{mnm2021} design an objective to jointly optimize the policy and the model in online MBRL.
\citet{vagram2022} propose a per-sample diagonal scaling matrix for training a deterministic model in online RL, which improves on \citet{vaml2017} and \citet{itervaml2018}.
In offline MBRL, \citet{gamps2020} propose a weighting scheme to learn the transition model.
This method requires the behavior policy and estimates the trajectory-wise importance-sampling weight, which is known to suffer from the ``curse of horizon'' \citep{breakingcurse2018}.
Thus, this method may not scale to complex high-dimensional offline RL tasks.
\citet{mml2021} propose a minimax objective for model learning, where the model is optimized over function classes of the MIW and the state-value function.
Most similar to our work, \citet{wmopo2021} also use a MIW-weighted MLE objective for model training. 
They use a different version of the MIW estimated via multiple \textit{full-trajectory} rollouts in the learned model.
This strategy may suffer from the inaccuracy of the learned model \citep{asi2012,mbmpo2018,metrpo2018} and can be time-consuming.
In this paper, we develop a fixed-point style method to train the MIW that only requires samples from the offline dataset and the initial state-distribution, which can be more stable and efficient. 

\myparagraph{Off-policy evaluation (OPE).}
OPE \citep{dualdice2019,gendice2020,jiang16doubly,coindice2020,feng2019kernel,feng2020accountable,feng2021non,tang2020off,feng2022offline} is a well-studied problem with several lines of research.
A natural way is trajectory-wise importance sampling (IS) \citep{opetrajectoryis2001}.
This, however, suffers from its variance, which can grow exponentially with the horizon length \citep{breakingcurse2018,minimaxope2015}.
Marginal importance sampling methods \citep{breakingcurse2018} are developed to address this problem by estimating the IS ratio between two stationary distributions, leading to estimators whose variance are polynomial \wrt\mbox{} horizon \citep{optope2019}.
Though prior works can also be used to estimate the MIW, 
they may not be directly applicable to our framework, since they either require knowing the behavior policy \citep{breakingcurse2018,tang2019doubly}, or 
need to assume the MIW lying in the space of RKHS \citep{blackope2020}.
Another approach to estimate the MIW is through saddle-point optimization \citep{dualdice2019,gendice2020,bestdice2020,opeduality2020}, which can suffer from its complex optimization problem.
In this paper, we derive an fixed-point-style and behavior-agnostic MIW estimation method, which is simple to train and does not require additional assumptions.

% \vspace{-.5em}
\section{Conclusion} \label{sec:conclusion}
In this paper, we are motivated by the mismatched model objectives in offline MBRL to design an iterative algorithm that alternates between model training and policy optimization.
Both the model and the policy are trained to maximize a lower bound of the true expected return.
The proposed new algorithm performs competitively with several SOTA baselines.

\myparagraph{Limitation.} 
Our current approach requires training two additional networks (discriminator $D_{\vpsi}$ and MIW $\omega$), and the model is (re)trained periodically to mitigate the objective mismatch,
which brings additional computational cost in order to obtain better empirical results.

\myparagraph{Future work.} 
We plan to conduct a theoretical analysis of our proposed framework, 
and investigate more efficient ways to unify model training and policy learning.

%\begin{ack}
\section*{Acknowledgments}
S. Yang, S. Zhang, and M. Zhou acknowledge the support of NSF IIS 1812699 and 2212418, and the Texas Advanced Computing Center (TACC) for providing HPC
resources that have contributed to the research results reported within this paper.
%\end{ack}

% \clearpage
% \section*{References}
% \bibliographystyle{icml2020}
\bibliographystyle{unsrtnat}
\bibliography{ref}

%%%%%%%%%%%%%%%%%%%%%%%%%%%%%%%%%%%%%%%%%%%%%%%%%%%%%%%%%%%%

\clearpage
\appendix

\begin{center}
\Large
\textbf{Appendix}
\end{center}

\section{Additional table} \label{sec:additional_tables}

Table~\ref{table:ablation_sec5} presents the numerical results for the ablation study in \Secref{sec:exp_ablation}.

\begin{table}[ht] 
\captionsetup{font=footnotesize}
\caption{
\footnotesize Normalized returns for experiments on the D4RL tasks.
Mean and standard deviation across five random seeds are reported.
NW denotes the No-Weights variant in \Secref{sec:exp_ablation} \textbf{(a)}.
VPM, GenDICE, and DualDICE denote the variants for the MIW estimation in \Secref{sec:exp_ablation} \textbf{(b)}.
WPR denotes the Weighted Policy Regularizer variant in \Secref{sec:exp_ablation} \textbf{(c)}.
The results of our main method in \Secref{sec:exp_main} is reported in column Main.
Here, ``hcheetah'' denotes ``halfcheetah'', ``med'' denotes ``medium'', ``rep'' denotes ``replay'', and ``exp'' denotes ``expert''.} 
\label{table:ablation_sec5} 
\centering 
\def\arraystretch{1.3}
\resizebox{1.\textwidth}{!}{
\begin{tabular}{lcccccc}
\toprule
\toprule
                 Task Name &                                NW &                               VPM &                           GenDICE &                          DualDICE &                               WPR &                              Main \\
\midrule
              maze2d-umaze &   31.3 $\pm$ {\footnotesize 25.3} &   57.8 $\pm$ {\footnotesize 29.1} &   46.2 $\pm$ {\footnotesize 23.1} &   42.6 $\pm$ {\footnotesize 19.4} &   14.4 $\pm$ {\footnotesize 21.4} &    55.8 $\pm$ {\footnotesize 3.9} \\
             maze2d-med &   85.4 $\pm$ {\footnotesize 28.6} &   91.1 $\pm$ {\footnotesize 54.3} &  101.3 $\pm$ {\footnotesize 50.3} &   74.8 $\pm$ {\footnotesize 42.6} &     5.0 $\pm$ {\footnotesize 8.3} &  107.2 $\pm$ {\footnotesize 45.0} \\
              maze2d-large &  131.4 $\pm$ {\footnotesize 63.0} &  150.2 $\pm$ {\footnotesize 64.3} &  136.1 $\pm$ {\footnotesize 63.0} &  149.2 $\pm$ {\footnotesize 32.7} &  104.3 $\pm$ {\footnotesize 48.1} &  180.0 $\pm$ {\footnotesize 39.3} \\
              \midrule
            Average Maze2D &                              82.7 &                              99.7 &                              94.5 &                              88.9 &                              41.2 &                             114.3 \\
        \midrule
        hcheetah-med &    51.5 $\pm$ {\footnotesize 0.3} &    50.8 $\pm$ {\footnotesize 0.7} &   33.5 $\pm$ {\footnotesize 18.8} &     3.5 $\pm$ {\footnotesize 1.8} &    50.1 $\pm$ {\footnotesize 0.4} &    51.7 $\pm$ {\footnotesize 0.4} \\
           walker2d-med &    77.3 $\pm$ {\footnotesize 6.0} &   63.7 $\pm$ {\footnotesize 28.3} &    78.7 $\pm$ {\footnotesize 5.3} &    79.7 $\pm$ {\footnotesize 4.7} &   41.8 $\pm$ {\footnotesize 18.9} &    83.1 $\pm$ {\footnotesize 1.8} \\
             hopper-med &    55.0 $\pm$ {\footnotesize 8.9} &    52.3 $\pm$ {\footnotesize 6.6} &   56.5 $\pm$ {\footnotesize 11.3} &    51.6 $\pm$ {\footnotesize 9.1} &   52.5 $\pm$ {\footnotesize 21.6} &    58.9 $\pm$ {\footnotesize 7.9} \\
 hcheetah-med-rep &    44.6 $\pm$ {\footnotesize 0.3} &    44.3 $\pm$ {\footnotesize 0.4} &    43.8 $\pm$ {\footnotesize 1.2} &   34.9 $\pm$ {\footnotesize 18.1} &    44.4 $\pm$ {\footnotesize 0.4} &    44.6 $\pm$ {\footnotesize 0.7} \\
    walker2d-med-rep &    74.6 $\pm$ {\footnotesize 4.2} &    73.5 $\pm$ {\footnotesize 7.3} &   37.8 $\pm$ {\footnotesize 32.6} &   66.9 $\pm$ {\footnotesize 29.6} &   55.8 $\pm$ {\footnotesize 14.3} &    81.5 $\pm$ {\footnotesize 3.0} \\
      hopper-med-rep &   75.4 $\pm$ {\footnotesize 19.1} &   16.6 $\pm$ {\footnotesize 20.0} &   87.7 $\pm$ {\footnotesize 16.6} &   79.7 $\pm$ {\footnotesize 19.8} &   40.1 $\pm$ {\footnotesize 11.1} &    91.1 $\pm$ {\footnotesize 9.5} \\
 hcheetah-med-exp &    88.0 $\pm$ {\footnotesize 5.1} &   54.0 $\pm$ {\footnotesize 43.2} &   79.0 $\pm$ {\footnotesize 18.6} &   19.4 $\pm$ {\footnotesize 35.9} &    80.1 $\pm$ {\footnotesize 5.9} &    90.2 $\pm$ {\footnotesize 2.0} \\
    walker2d-med-exp &    99.1 $\pm$ {\footnotesize 7.9} &   90.3 $\pm$ {\footnotesize 30.2} &   68.3 $\pm$ {\footnotesize 32.3} &   103.9 $\pm$ {\footnotesize 4.3} &   62.9 $\pm$ {\footnotesize 17.6} &   107.7 $\pm$ {\footnotesize 2.1} \\
      hopper-med-exp &   41.7 $\pm$ {\footnotesize 16.9} &    69.8 $\pm$ {\footnotesize 9.6} &   54.7 $\pm$ {\footnotesize 22.5} &   56.2 $\pm$ {\footnotesize 26.5} &    11.3 $\pm$ {\footnotesize 7.6} &   76.0 $\pm$ {\footnotesize 15.7} \\
            \midrule
            Average MuJoCo &                              67.5 &                              57.3 &                                60.0 &                              55.1 &                              48.8 &                                76.1 \\
                 \midrule
                 pen-human &    11.4 $\pm$ {\footnotesize 8.6} &   13.8 $\pm$ {\footnotesize 12.1} &    15.7 $\pm$ {\footnotesize 9.7} &   16.4 $\pm$ {\footnotesize 11.9} &   24.8 $\pm$ {\footnotesize 14.7} &   20.6 $\pm$ {\footnotesize 10.7} \\
                pen-cloned &   35.3 $\pm$ {\footnotesize 15.7} &   39.5 $\pm$ {\footnotesize 15.8} &   34.4 $\pm$ {\footnotesize 17.2} &   37.3 $\pm$ {\footnotesize 16.4} &   35.5 $\pm$ {\footnotesize 13.8} &   57.4 $\pm$ {\footnotesize 19.2} \\
                pen-exp &  129.6 $\pm$ {\footnotesize 15.4} &   98.1 $\pm$ {\footnotesize 33.7} &  128.6 $\pm$ {\footnotesize 16.2} &  130.4 $\pm$ {\footnotesize 16.0} &  102.0 $\pm$ {\footnotesize 36.3} &   138.3 $\pm$ {\footnotesize 6.6} \\
               door-exp &   73.5 $\pm$ {\footnotesize 30.9} &   39.3 $\pm$ {\footnotesize 45.2} &   83.1 $\pm$ {\footnotesize 25.9} &   68.0 $\pm$ {\footnotesize 35.0} &   20.7 $\pm$ {\footnotesize 16.4} &    96.3 $\pm$ {\footnotesize 9.8} \\
            \midrule
            Average Adroit &                              62.4 &                              47.7 &                              65.4 &                                63.0 &                              45.8 &                              78.2 \\
               \midrule
               Average All &                              69.1 &                              62.8 &                              67.8 &                              63.4 &                              46.6 &                              83.8 \\
\bottomrule
\bottomrule
\end{tabular}
}
\end{table}

Table~\ref{table:additional_ablation} provides additional ablation study on several building blocks of our main method.
In Table~\ref{table:additional_ablation}, we test the following variants to demonstrate the effectiveness of our framework.

\begin{itemize}[leftmargin=*]
    \item \textbf{Main} denotes the results of our main method in \Secref{sec:exp_main}.
    \item \textbf{NW} denotes the No-Weights variant in \Secref{sec:exp_ablation} \textbf{(a)}.
    \item \textbf{WPR} denotes the Weighted Policy Regularizer variant in \Secref{sec:exp_ablation} \textbf{(c)}.
    \item  \textbf{KL-Dual} denotes the variant changing the JSD regularizer in Main to the $\mathrm{KL}(P^*(s'\given s,a) \pi_b(a'\given s') \,||\, \widehat P(s' \given s,a)\pi(a'\given s'))$ term in Theorem~\ref{thm:eval_error_upper_bound}, where $(s,a) \sim \dbgtrue$ and the KL term is estimated via the dual representation $\mathrm{KL}(p \,||\, q) = \sup_{T}\cbr{ \E_{p}[T] - \log(\E_q[e^T])}$ \citep{kldual1983}.
    \item \textbf{KL-Dual+WPR} denotes the variant changing the JSD regularizer in Main to the weighted $\mathrm{KL}$ term $D_{\pi}(P^*, \widehat P)$ in Theorem~\ref{thm:eval_error_upper_bound}.
    \item \textbf{Gaussian} denotes the variant changing the implicit policy in Main to the Gaussian policy.
    \item \textbf{No Reg.} denotes the variant removing the proposed regularizer in the policy learning.
    \item \textbf{No model-rollout} denotes the variant of no rollout data in policy learning. 
    \item \textbf{Rew. Test} denotes the variant of using the estimated reward function as the test function when training the MIW $\omega$.
\end{itemize}

Apart from the discussion in \Secref{sec:exp_ablation} \textbf{(a)} and \Secref{sec:exp_ablation} \textbf{(c)} on the \textbf{NW} and the \textbf{WPR} variants.
We see that changing the proposed JSD regularizer in \Secref{sec:policy_learning} to the KL-dual-based regularizers breaks the policy learning. 
This may be related to the unstable estimation of KL-dual discussed in \Secref{sec:policy_learning}.
Changing the implicit policy to the Gaussian policy generally leads to worse performance.
The performance difference is especially significant on the Maze2D and Adroit datasets.
This aligns with the observation in \citet{jointmatching2022} that a uni-model Gaussian policy may be insufficient to capture the necessary multiple action modes in the Maze2D and Adroit datasets, while an implicit policy is likely to be capable.
Removing the proposed regularization term breaks the policy learning, showing the efficacy of the proposed regularizer.
Removing rollout data in the policy learning generally leads to worse performance and larger standard deviations.
This shows the benefit of adding the model-rollout data, that they can mitigate the off-policy issue by taking into account the rollouts of the learned policy.
Finally, compared with the main method that uses the action-value function as the test function, the \textbf{Rew. Test} variant generally has worse performance and larger standard deviations.
As discussed in Appendix~\ref{sec:miw_diss}, using the (estimated) reward function as the test function loses the nice mathematical properties, and the resulting objective for training the MIW $\omega$ may not be easy to optimize.

\begin{table}[tb] 
\captionsetup{font=footnotesize}
\caption{
\footnotesize Normalized returns for experiments on the D4RL tasks.
Mean and standard deviation across five random seeds are reported.
\\
\textbf{Main} denotes the results of our main method in \Secref{sec:exp_main}.
\textbf{NW} denotes the No-Weights variant in \Secref{sec:exp_ablation} \textbf{(a)}.
\textbf{WPR} denotes the Weighted Policy Regularizer variant in \Secref{sec:exp_ablation} \textbf{(c)}.
\textbf{KL-Dual} denotes the variant changing the JSD regularizer in Main to the $\mathrm{KL}$ term in Theorem~\ref{thm:eval_error_upper_bound}.
\textbf{KL-Dual+WPR} denotes the variant changing the JSD regularizer in Main to the weighted $\mathrm{KL}$ term $D_{\pi}(P^*, \widehat P)$ in Theorem~\ref{thm:eval_error_upper_bound}.
\textbf{Gaussian} denotes the variant changing the implicit policy in Main to the Gaussian policy.
\textbf{No Reg.} denotes the variant removing the proposed regularizer in the policy learning.
\textbf{No model-rollout} denotes the variant of no rollout data in policy learning. 
\textbf{Rew. Test} denotes the variant of using the estimated reward function as the test function when training the MIW $\omega$.
Here, ``hcheetah'' denotes ``halfcheetah,'' ``med'' denotes ``medium,'' ``rep'' denotes ``replay,'' and ``exp'' denotes ``expert.''} 
\label{table:additional_ablation} 
\centering 
\def\arraystretch{1.3}
\resizebox{1.\textwidth}{!}{
\begin{tabular}{lccccccccc}
\toprule
\toprule
        Task Name &                              Main &                                NW &                               WPR &                          KL-Dual &                      KL-Dual+WPR &               Gaussian &                          No Reg. &                  No model-rollout & Rew. Test \\
\midrule
     maze2d-umaze &    55.8 $\pm$ {\footnotesize 3.9} &   31.3 $\pm$ {\footnotesize 25.3} &   14.4 $\pm$ {\footnotesize 21.4} &  -3.3 $\pm$ {\footnotesize 18.6} &  -0.9 $\pm$ {\footnotesize 16.5} &   41.5 $\pm$ {\footnotesize 26.3} &  -13.9 $\pm$ {\footnotesize 1.3} &   48.4 $\pm$ {\footnotesize 14.7} & 34.9 $\pm$ {\footnotesize  24.3} \\
    maze2d-medium &  107.2 $\pm$ {\footnotesize 45.0} &   85.4 $\pm$ {\footnotesize 28.6} &     5.0 $\pm$ {\footnotesize 8.3} &  16.9 $\pm$ {\footnotesize 17.9} &   4.8 $\pm$ {\footnotesize 23.2} &   59.5 $\pm$ {\footnotesize 29.5} &  83.3 $\pm$ {\footnotesize 30.1} &  120.7 $\pm$ {\footnotesize 41.9} & 85.6 $\pm$ {\footnotesize  45} \\
     maze2d-large &  180.0 $\pm$ {\footnotesize 39.3} &  131.4 $\pm$ {\footnotesize 63.0} &  104.3 $\pm$ {\footnotesize 48.1} &   -1.0 $\pm$ {\footnotesize 4.0} &   -0.5 $\pm$ {\footnotesize 7.4} &  106.1 $\pm$ {\footnotesize 51.0} &    0.6 $\pm$ {\footnotesize 0.7} &  121.0 $\pm$ {\footnotesize 26.5} & 143.8 $\pm$ {\footnotesize 44}  \\
     \midrule
   Average Maze2D &                             114.3 &                              82.7 &                              41.2 &                              4.2 &                              1.1 &                                69.0 &                             23.3 &                              96.7 & 88.1 \\
   \midrule
     hcheetah-med &    51.7 $\pm$ {\footnotesize 0.4} &    51.5 $\pm$ {\footnotesize 0.3} &    50.1 $\pm$ {\footnotesize 0.4} &  16.2 $\pm$ {\footnotesize 13.0} &   28.3 $\pm$ {\footnotesize 2.9} &    49.2 $\pm$ {\footnotesize 0.4} &  30.5 $\pm$ {\footnotesize 26.3} &    51.9 $\pm$ {\footnotesize 0.6} & 51.5  $\pm$ {\footnotesize 0.6} \\
     walker2d-med &    83.1 $\pm$ {\footnotesize 1.8} &    77.3 $\pm$ {\footnotesize 6.0} &   41.8 $\pm$ {\footnotesize 18.9} &    0.3 $\pm$ {\footnotesize 1.5} &    2.8 $\pm$ {\footnotesize 3.4} &    72.5 $\pm$ {\footnotesize 6.7} &    3.5 $\pm$ {\footnotesize 4.4} &    83.3 $\pm$ {\footnotesize 4.4} & 77.8 $\pm$ {\footnotesize 8.4}  \\
       hopper-med &    58.9 $\pm$ {\footnotesize 7.9} &    55.0 $\pm$ {\footnotesize 8.9} &   52.5 $\pm$ {\footnotesize 21.6} &    7.6 $\pm$ {\footnotesize 3.4} &    6.9 $\pm$ {\footnotesize 7.0} &    52.1 $\pm$ {\footnotesize 3.7} &    6.2 $\pm$ {\footnotesize 9.0} &   42.8 $\pm$ {\footnotesize 22.8} & 57.3 $\pm$ {\footnotesize 9.7}  \\
 hcheetah-med-rep &    44.6 $\pm$ {\footnotesize 0.7} &    44.6 $\pm$ {\footnotesize 0.3} &    44.4 $\pm$ {\footnotesize 0.4} &   23.1 $\pm$ {\footnotesize 5.9} &   18.9 $\pm$ {\footnotesize 8.0} &    41.1 $\pm$ {\footnotesize 0.7} &   67.3 $\pm$ {\footnotesize 2.7} &    41.9 $\pm$ {\footnotesize 1.7} & 44.6 $\pm$ {\footnotesize 0.2} \\
 walker2d-med-rep &    81.5 $\pm$ {\footnotesize 3.0} &    74.6 $\pm$ {\footnotesize 4.2} &   55.8 $\pm$ {\footnotesize 14.3} &    2.3 $\pm$ {\footnotesize 4.4} &    1.1 $\pm$ {\footnotesize 2.5} &    54.1 $\pm$ {\footnotesize 5.8} &   10.6 $\pm$ {\footnotesize 8.9} &   13.8 $\pm$ {\footnotesize 13.4} & 72.8 $\pm$ {\footnotesize 10.4}  \\
   hopper-med-rep &    91.1 $\pm$ {\footnotesize 9.5} &   75.4 $\pm$ {\footnotesize 19.1} &   40.1 $\pm$ {\footnotesize 11.1} &    6.0 $\pm$ {\footnotesize 5.7} &    5.5 $\pm$ {\footnotesize 7.3} &   71.3 $\pm$ {\footnotesize 19.7} &    9.9 $\pm$ {\footnotesize 6.5} &   85.8 $\pm$ {\footnotesize 16.8} & 79.3 $\pm$ {\footnotesize 20.6}  \\
 hcheetah-med-exp &    90.2 $\pm$ {\footnotesize 2.0} &    88.0 $\pm$ {\footnotesize 5.1} &    80.1 $\pm$ {\footnotesize 5.9} &  20.7 $\pm$ {\footnotesize 11.1} &   19.1 $\pm$ {\footnotesize 9.4} &    85.0 $\pm$ {\footnotesize 2.9} &   7.5 $\pm$ {\footnotesize 12.6} &    79.0 $\pm$ {\footnotesize 7.2} & 88.5 $\pm$ {\footnotesize 4.5}  \\
 walker2d-med-exp &   107.7 $\pm$ {\footnotesize 2.1} &    99.1 $\pm$ {\footnotesize 7.9} &   62.9 $\pm$ {\footnotesize 17.6} &    0.3 $\pm$ {\footnotesize 1.0} &    1.3 $\pm$ {\footnotesize 2.7} &    90.6 $\pm$ {\footnotesize 6.6} &    2.6 $\pm$ {\footnotesize 0.8} &   104.7 $\pm$ {\footnotesize 3.4} & 103 $\pm$ {\footnotesize 6.9}  \\
   hopper-med-exp &   76.0 $\pm$ {\footnotesize 15.7} &   41.7 $\pm$ {\footnotesize 16.9} &    11.3 $\pm$ {\footnotesize 7.6} &    3.1 $\pm$ {\footnotesize 1.4} &    4.5 $\pm$ {\footnotesize 3.3} &   59.1 $\pm$ {\footnotesize 25.5} &   12.4 $\pm$ {\footnotesize 9.6} &   48.7 $\pm$ {\footnotesize 17.4} & 59.8 $\pm$ {\footnotesize 26}  \\
   \midrule
   Average MuJoCo &                                76.1 &                              67.5 &                              48.8 &                              8.8 &                              9.8 &                              63.9 &                             16.7 &                              61.3 & 70.5 \\
   \midrule
        pen-human &   20.6 $\pm$ {\footnotesize 10.7} &    11.4 $\pm$ {\footnotesize 8.6} &   24.8 $\pm$ {\footnotesize 14.7} &    9.3 $\pm$ {\footnotesize 4.0} &    4.7 $\pm$ {\footnotesize 7.6} &     8.5 $\pm$ {\footnotesize 5.6} &    1.6 $\pm$ {\footnotesize 5.4} &     8.4 $\pm$ {\footnotesize 9.0} & 12.2 $\pm$ {\footnotesize 10.6}  \\
       pen-cloned &   57.4 $\pm$ {\footnotesize 19.2} &   35.3 $\pm$ {\footnotesize 15.7} &   35.5 $\pm$ {\footnotesize 13.8} &  22.4 $\pm$ {\footnotesize 12.5} &  35.7 $\pm$ {\footnotesize 12.8} &   39.3 $\pm$ {\footnotesize 16.4} &  39.1 $\pm$ {\footnotesize 21.3} &     3.5 $\pm$ {\footnotesize 6.2} & 37.5 $\pm$ {\footnotesize 16.6}  \\
          pen-exp &   138.3 $\pm$ {\footnotesize 6.6} &  129.6 $\pm$ {\footnotesize 15.4} &  102.0 $\pm$ {\footnotesize 36.3} &  31.2 $\pm$ {\footnotesize 15.5} &  30.1 $\pm$ {\footnotesize 18.3} &  102.0 $\pm$ {\footnotesize 14.9} &    5.2 $\pm$ {\footnotesize 7.6} &  109.2 $\pm$ {\footnotesize 40.9} & 133 $\pm$ {\footnotesize 12.7}  \\
         door-exp &    96.3 $\pm$ {\footnotesize 9.8} &   73.5 $\pm$ {\footnotesize 30.9} &   20.7 $\pm$ {\footnotesize 16.4} &    0.1 $\pm$ {\footnotesize 1.4} &    0.2 $\pm$ {\footnotesize 0.4} &   20.3 $\pm$ {\footnotesize 21.2} &   -0.3 $\pm$ {\footnotesize 0.1} &   53.9 $\pm$ {\footnotesize 41.4} & 86 $\pm$ {\footnotesize 24.3}  \\
         \midrule
   Average Adroit &                              78.2 &                              62.4 &                              45.8 &                             15.8 &                             17.7 &                              42.5 &                             11.4 &                              43.8 & 67.2 \\
   \midrule
      Average All &                              83.8 &                              69.1 &                              46.6 &                              9.7 &                             10.2 &                              59.5 &                             16.6 &                              63.6 & {73.0} \\
\bottomrule
\bottomrule
\end{tabular}
}
\end{table}

\section{Proofs} \label{sec:proofs}

\subsection{Preliminary on discounted stationary state-action distribution}

Recall that the discounted visitation frequency for a policy $\pi$ on MDP $\gM$ with transition $P^*$ is defined as 
\begin{equation*}\textstyle
    \dpigtrue \br{s, a} \triangleq (1-\gamma) \sum_{t=0}^\infty \gamma^t \Pr\br{s_t = s, a_t = a \given \mu_0, \pi, P^*}.
\end{equation*}
Denote $\mT_{\pi}^* (s' \given s) = p_{\pi}\left(s_{t+1}=s' \given s_t = s \right) = \sum_{a \in \sA}P^*\left(s_{t+1}=s' \given s_t=s, a_t=a \right) \pi\left(a \given s\right)$.

From \citet{breakingcurse2018}, Lemma 3, for $\dpigtrue(s)$ we have, 
\begin{equation*}
    \gamma \sum_s \mT_{\pi}^*(s' \given s) \dpigtrue(s) - \dpigtrue(s') + (1-\gamma) \mu_0(s') = 0, \quad \forall \, s' \in \sS,
\end{equation*}
where $\mu_0$ is the initial-state distribution. 
Multiply $\pi(a'\given s')$ on both sides, we get, $\forall \, (s', a') \in \sS \times \sA$,
\begin{equation*}
    \begin{split}
    & \gamma \sum_s \mT_{\pi}^*(s' \given s) \dpigtrue(s) \pi(a'\given s') - \dpigtrue(s',a') + (1-\gamma) \mu_0(s')\pi(a'\given s') = 0 \\
    \iff & \gamma \sum_{s,a} P^*\left(s' \given s, a \right) \dpigtrue(s,a) \pi(a'\given s') - \dpigtrue(s',a') + (1-\gamma) \mu_0(s')\pi(a'\given s') = 0
    \end{split}
\end{equation*}
We can also multiply $\widetilde \pi(a'\given s')$ on both sides to get, $\forall \, (s', a') \in \sS \times \sA$,
\begin{equation*}
    \gamma \sum_s \mT_{\pi}^*(s' \given s) \dpigtrue(s) \widetilde\pi(a'\given s') - \dpigtrue(s') \widetilde\pi(a'\given s') + (1-\gamma) \mu_0(s') \widetilde\pi(a'\given s') = 0.
\end{equation*}
Denote $\dpigtrue\br{s, a, s'} \triangleq \dpigtrue(s) \pi(a \given s)P^*(s'\given s, a )$. 
For any integrable function $f(s, a)$, we multiply both sides of the above equation by $f(s', a')$ and summing over $s', a'$, we get
\begin{equation*}
        \gamma \E_{\substack{\br{s, a, s'} \sim \dpigtrue \\ a'\sim \widetilde\pi\br{\cdot \given s'}} } \sbr{f\br{s',a'}} - \E_{\substack{s\sim \dpigtrue \\ a\sim \widetilde\pi(\cdot \given s)}}\sbr{f\br{s,a}}  + (1-\gamma) \E_{\substack{s\sim \mu_0 \\ a \sim \widetilde\pi(\cdot \given s)}} \sbr{f\br{s, a}} = 0.
    \end{equation*}
For any given bounded function $g(s, a)$, define function $f$ to satisfy
\begin{equation*}
    f(s, a) = g(s, a) + \gamma \E_{\substack{s' \sim P^*(\cdot \given s, a), a' \sim \pi(\cdot \given s')}} \sbr{f(s', a')}, \forall\, s, a,
\end{equation*}
then we have
\begin{equation*}
    \begin{split}
        \E_{s \sim \mu_0(\cdot), a \sim \pi(\cdot \given s)}\sbr{f\br{s, a}} &= \E_{s \sim \mu_0(\cdot), a \sim \pi(\cdot \given s)}\sbr{ g(s, a) + \gamma \E_{s' \sim P^*(\cdot \given s, a), a' \sim \pi(\cdot \given s')} \sbr{f(s', a')} } \\
        &= \E\sbr{\sum_{t=0}^\infty \gamma^t g\br{s_t, a_t} \given s_0 \sim \mu_0(\cdot), a_t \sim \pi(\cdot \given s_t), s_{t+1} \sim P^*\br{\cdot \given s_t, a_t} } \\
        &= \br{1 - \gamma}^{-1} \E_{\br{s, a} \sim \dpigtrue}\sbr{g\br{s, a}},
    \end{split}
\end{equation*}
based on the definition of $\dpigtrue(s, a)$ stated above.

Indeed, $f\br{s, a}$ is the action-value function of policy $\pi$ under the reward $g(s, a)$ on MDP $\gM$, and can be approximated using neural network under the classical regularity assumptions on $g(s,a)$ and $\gM$.
Also, under these classical regularity conditions, reward function $g$ and its action-value function $f$ have one-to-one correspondence, with $f$ being the unique solution to the Bellman equation.

\subsection{Proofs of Theorem~\ref{thm:eval_error_upper_bound}}\label{sec:proof_eval_error_upper_bound}
The following Lemma will be used in the proof of Theorem~\ref{thm:eval_error_upper_bound}.
\begin{lemma} \label{thm:KL_conditional}
For any $\widehat P(s' \given s,a), \pi(a'\given s'), P^*(s'\given s,a), \pi_b(a'\given s')$, we have
\begin{equation*}
    \mathrm{KL}\br{ P^*(s'\given s,a) \pi_b(a'\given s') \,||\, \widehat P(s' \given s,a)\pi(a'\given s')} \geq \mathrm{KL}\br{P^*(s'\given s,a) \,||\, \widehat P(s' \given s,a)}.
\end{equation*}
\end{lemma}
\begin{proof}
Since $\mathrm{KL}(\cdot \, || \, \cdot) \geq 0$, we have
\begin{equation*}
\resizebox{0.94\textwidth}{!}{%
$
\begin{aligned}
    & \mathrm{KL}\br{P^*(s'\given s,a) \pi_b(a'\given s')  \,||\, \widehat P(s' \given s,a)\pi(a'\given s')}\\
    =&\int  P^*(s'\given s,a) \pi_b(a'\given s') \log \frac{P^*(s'\given s,a) \pi_b(a'\given s')}{\widehat P(s' \given s,a)\pi(a'\given s')} \diff (s',a') \\
    =& \int  P^*(s'\given s,a) \pi_b(a'\given s') \br{
    \log \frac{P^*(s'\given s,a)}{\widehat P(s' \given s,a)} + \log\frac{\pi_b(a'\given s')}{\pi(a'\given s')}} \diff (s',a') \\
    =& \int  P^*(s'\given s,a) \pi_b(a'\given s') \log\frac{\pi_b(a'\given s')}{\pi(a'\given s')} \diff (s',a') + \int  P^*(s'\given s,a) \log \frac{P^*(s'\given s,a)}{\widehat P(s' \given s,a)} \diff s' \\
    =& \E_{s' \sim  P^*(\cdot \given s,a)}\sbr{\mathrm{KL}\br{\pi_b(a'\given s')\,||\,\pi(a'\given s')}} + \mathrm{KL}\br{P^*(s' \given s,a) \, || \, \widehat P(s' \given s,a)} \\
    \geq& \mathrm{KL}\br{P^*(s' \given s,a) \, || \, \widehat P(s' \given s,a)},
\end{aligned}
$%
}    
\end{equation*}
as desired.
\end{proof}
\begin{remark}\label{thm:KL_conditional_s_a}
It also holds that
\begin{equation*}
    \begin{split}
        &\E_{(s,a) \sim \dbgtrue}\sbr{\mathrm{KL}\br{P^*(s'\given s,a) \pi_b(a'\given s') \,||\, \widehat P(s' \given s,a)\pi(a'\given s')}} \\
        \leq & \mathrm{KL}\br{P^*(s'\given s,a) \pi_b(a'\given s') \dbgtrue(s,a) \,||\, \widehat P(s' \given s,a)\pi(a'\given s')\dbgtrue(s) \pi(a\given s)},
    \end{split}
\end{equation*}
by using similar proof steps.
\end{remark}

\begin{proof}[Proof of Theorem~\ref{thm:eval_error_upper_bound}]
\begin{equation} \label{eq:eval_error_policy_model}
\resizebox{0.93\textwidth}{!}{%
    $
    \begin{aligned}
& \left\vert J(\pi,\widehat{P})-J(\pi,P^{*})\right\vert \\
 =& \left\vert(1-\gamma)\mathbb{E}_{\substack{s\sim\mu_{0}\\a\sim\pi(\cdot|s)}}\left[Q_{\pi}^{\widehat{P}}(s,a)\right]-\mathbb{E}_{d_{\pi,\gamma}^{P^{*}}}\left[Q_{\pi}^{\widehat{P}}(s,a)\right]+\mathbb{E}_{d_{\pi,\gamma}^{P^{*}}}\left[Q_{\pi}^{\widehat{P}}(s,a)\right]-\mathbb{E}_{(s,a)\sim d_{\pi,\gamma}^{P^{*}}}\left[r(s,a)\right]\right\vert\\
 =& \left\vert (1-\gamma)\mathbb{E}_{\substack{s\sim\mu_{0}\\a\sim\pi(\cdot|s)}}\left[Q_{\pi}^{\widehat{P}}(s,a)\right]-\mathbb{E}_{d_{\pi,\gamma}^{P^{*}}}\left[Q_{\pi}^{\widehat{P}}(s,a)\right]+\mathbb{E}_{(s,a)\sim d_{\pi,\gamma}^{P^{*}}}\left[Q_{\pi}^{\widehat{P}}(s,a)-r(s,a)\right]\right\vert \\
 =& \left\vert (1-\gamma)\mathbb{E}_{\substack{s\sim\mu_{0}\\a\sim\pi(\cdot|s)}}\left[Q_{\pi}^{\widehat{P}}(s,a)\right]-\mathbb{E}_{d_{\pi,\gamma}^{P^{*}}}\left[Q_{\pi}^{\widehat{P}}(s,a)\right]+\gamma\mathbb{E}_{(s,a)\sim d_{\pi,\gamma}^{P^{*}}}\left[\mathbb{E}_{\substack{s^{\prime}\sim\widehat{P}(\cdot|s,a)\\a^{\prime}\sim\pi(\cdot|s^{\prime})}}\left[Q_{\pi}^{\widehat{P}}(s^{\prime},a^{\prime})\right]\right]\right\vert \\
 =& \left\vert -\gamma\mathbb{E}_{(s,a)\sim d_{\pi,\gamma}^{P^{*}}}\left[\mathbb{E}_{\substack{s^{\prime}\sim P^{*}(\cdot|s,a) \\ a^{\prime}\sim\pi(\cdot|s^{\prime})}}\left[Q_{\pi}^{\widehat{P}}(s^{\prime},a^{\prime})\right]\right]+\gamma\mathbb{E}_{(s,a)\sim d_{\pi,\gamma}^{P^{*}}}\left[\mathbb{E}_{\substack{s^{\prime}\sim\widehat{P}(\cdot|s,a)\\a^{\prime}\sim\pi(\cdot|s^{\prime})}}\left[Q_{\pi}^{\widehat{P}}(s^{\prime},a^{\prime})\right]\right]\right\vert \\
 =& \gamma\left\vert \mathbb{E}_{(s,a)\sim d_{\pi,\gamma}^{P^{*}}}\left[\mathbb{E}_{\substack{s^{\prime}\sim P^{*}(\cdot|s,a)\\a^{\prime}\sim\pi(\cdot|s^{\prime})}}\left[Q_{\pi}^{\widehat{P}}(s^{\prime},a^{\prime})\right]-\mathbb{E}_{\substack{s^{\prime}\sim\widehat{P}(\cdot|s,a)\\a^{\prime}\sim\pi(\cdot|s^{\prime})}}\left[Q_{\pi}^{\widehat{P}}(s^{\prime},a^{\prime})\right]\right]\right\vert \\
 =& \gamma\left\vert \mathbb{E}_{(s,a)\sim d_{\pi,\gamma}^{P^{*}}}\left[\mathbb{E}_{s^{\prime}\sim P^{*}(\cdot|s,a)}\left[V_{\pi}^{\widehat{P}}(s^{\prime})\right]-\mathbb{E}_{s^{\prime}\sim\widehat{P}(\cdot|s,a)}\left[V_{\pi}^{\widehat{P}}(s^{\prime})\right]\right]\right\vert \\
 \leq & \gamma \mathbb{E}_{(s,a)\sim d_{\pi,\gamma}^{P^{*}}}\left\{ \left\vert\mathbb{E}_{s^{\prime}\sim P^{*}(\cdot|s,a)}\left[V_{\pi}^{\widehat{P}}(s^{\prime})\right]-\mathbb{E}_{{s'}\sim\widehat{P}(\cdot|s,a)}\left[V_{\pi}^{\widehat{P}}({s'})\right]\right\vert\right\} \\
 \leq & \frac{\gamma \cdot r_{\max}}{1-\gamma} \mathbb{E}_{(s,a)\sim d_{\pi,\gamma}^{P^{*}}}\left\{ \sup_{V \in \mathcal{V}} \left\vert\mathbb{E}_{s^{\prime}\sim P^{*}(\cdot|s,a)}\left[V(s^{\prime})\right]-\mathbb{E}_{{s'}\sim\widehat{P}(\cdot|s,a)}\left[V({s'})\right]\right\vert\right\} \\
 = & \frac{\gamma \cdot r_{\max}}{1-\gamma} \mathbb{E}_{(s,a)\sim d_{\pi,\gamma}^{P^{*}}}\left[ \mathrm{TV} \left(P^{*}(\cdot|s,a) || \widehat{P}(\cdot|s,a)\right)\right] \\
 \leq & \frac{\gamma \cdot r_{\max}}{1-\gamma} \mathbb{E}_{(s,a)\sim d_{\pi,\gamma}^{P^{*}}}\left[ \sqrt{\frac{1}{2}\mathrm{KL} \left(P^{*}(\cdot|s,a) || \widehat{P}(\cdot|s,a)\right) }\right] \\
 \leq & \frac{\gamma \cdot r_{\max}}{\sqrt{2}\cdot (1-\gamma)} \sqrt{\mathbb{E}_{(s,a)\sim d_{\pi,\gamma}^{P^{*}}}\left[\mathrm{KL} \left(P^{*}(\cdot|s,a) || \widehat{P}(\cdot|s,a)\right)\right] } \\
 \leq& \frac{\gamma \cdot \rmax}{\sqrt{2}(1-\gamma)}
\sqrt{\E_{(s,a) \sim \dpigtrue}\sbr{\mathrm{KL}\br{P^*(s'\given s,a) \pi_b(a'\given s') \,||\, \widehat P(s' \given s,a)\pi(a'\given s') }}} \\
 =& \frac{\gamma \cdot \rmax}{\sqrt{2}(1-\gamma)}
\sqrt{\E_{(s,a) \sim \dbgtrue}\sbr{\omega(s,a) \mathrm{KL}\br{P^*(s'\given s,a) \pi_b(a'\given s') \,||\, \widehat P(s' \given s,a)\pi(a'\given s') }}}
\end{aligned}
$%
}
\end{equation}
where $\gV$ is the set of functions bounded by $1$, $\omega(s,a) \triangleq \frac{d_{\pi,\gamma}^{P^{*}}(s,a)}{\dbgtrue(s,a)}$
is the marginal importance weight (MIW).
Here we use the assumption that $\forall\, s,a, \, |r(s,a)| \leq r_{\max}$, and hence $|V_{\pi}^{\widehat{P}}(\cdot)|\leq \frac{r_{\max}}{1-\gamma}$.
We use Lemma~\ref{thm:KL_conditional} to introduce $\pi$ and $\pi_b$ into $\mathrm{KL}\br{\cdot \,||\, \cdot}$.
\end{proof}

\subsection{Derivation of Eq.~\eqref{eq:approx_new_reg}}\label{sec:derive_D_tilde}
\begin{proof}[Derivarion of Eq.~\eqref{eq:approx_new_reg}]
We follow the literature \citep[\eg,][]{errapi2003,Chen2019InformationTheoreticCI} to assume that $\forall\, s,a,\, \omega(s,a) \leq \omega_{\max}$
for some unknown finite constant $\omega_{\max}$.
Using Remark~\ref{thm:KL_conditional_s_a} and this assumption, we have
\begin{equation*}
    \resizebox{\textwidth}{!}{%
    $
    \begin{aligned}
        D_{\pi} (P^{*}, \widehat P) =&~ \E_{(s,a) \sim \dbgtrue}\sbr{\omega(s,a) \mathrm{KL}\br{P^*(s'\given s,a) \pi_b(a'\given s') \,||\, \widehat P(s' \given s,a)\pi(a'\given s') }}\\
        \leq &~ \omega_{\max} \cdot \E_{(s,a) \sim \dbgtrue}\sbr{\mathrm{KL}\br{P^*(s'\given s,a) \pi_b(a'\given s') \,||\, \widehat P(s' \given s,a)\pi(a'\given s') }}\\
\leq &~ 
        \omega_{\max}\cdot \mathrm{KL}\br{P^*(s'\given s,a) \pi_b(a'\given s')\dbgtrue(s,a) \,||\, \widehat P(s' \given s,a)\pi(a'\given s') \dbgtrue(s) \pi(a\given s)}\\
\approx &~ 
        \omega_{\max}\cdot \mathrm{JSD}\br{P^*(s'\given s,a) \pi_b(a'\given s')\dbgtrue(s,a) \,||\, \widehat P(s' \given s,a)\pi(a'\given s') \dbgtrue(s) \pi(a\given s)}\\
        \triangleq &~ \omega_{\max} \cdot \widetilde{D}_{\pi}(P^{*}, \widehat P),
    \end{aligned}
$%
}
\end{equation*}
where we approximate the KL divergence by the JSD.
\end{proof}

\subsection{Derivation of Eq.~\eqref{eq:dr_target_derive}}\label{sec:derive_dr_target}
\begin{proof}[Derivarion of Eq.~\eqref{eq:dr_target_derive}]
We have, $\forall\, (s,a) \in \mathbb{S} \times \mathbb{A}$,
\begin{equation}
    \begin{split}
        & \dbgtrue(s,a) \cdot \omega(s,a) = d_{\pi,\gamma}^{P^{*}}(s,a)  \iff \\
        & \dbgtrue(s',a') \cdot \omega(s',a') = d_{\pi,\gamma}^{P^{*}}(s',a') \\
        =&  \gamma \sum_s T_\pi^*(s' \given s) d_{\pi,\gamma}^{P^{*}}(s) \pi(a' \given s') +  (1-\gamma) \mu_0(s') \pi(a'  \given  s')\\
        =& \gamma \sum_{s,a} P^*(s'  \given  s,a) d_{\pi,\gamma}^{P^{*}}(s,a) \pi(a' \given s') +  (1-\gamma) \mu_0(s') \pi(a'  \given  s')  \\
        =& \gamma \sum_{s,a} P^*(s'  \given  s,a) \omega(s,a) \dbgtrue(s,a) \pi(a' \given s')  +  (1-\gamma) \mu_0(s') \pi(a'  \given  s'),
    \end{split}
\end{equation}
where $T_\pi^*(s' \given s) = \sum_{a \in \mathbb{A}} P^*(s' \given s,a) \pi(a \given s)$ is the state-transition kernel.
Here we change $(s,a)$ into $(s',a')$ in the ``$\iff$'' for notation simplicity.
\end{proof}

\subsection{Proof of Proposition~\ref{thm:miw_iterate_coverge}}\label{sec:proof_miw_iterate_coverge}
\begin{proof}[Proof of Proposition~\ref{thm:miw_iterate_coverge}]
With the assumed conditions, we would like to show that $\gT$ is a contraction mapping under $\| \cdot \|_\infty$, {\itshape i.e.}, $\|\gT\omega - \gT u\|_\infty \leq c \cdot \| \omega - u\|_\infty$, for all MIWs $\omega, u: \mathbb{S} \times \mathbb{A} \rightarrow \mathbb{R}$, for some constant $c < 1$.
We have,
\begin{equation}\label{eq:contraction_derive}
    \resizebox{.93\textwidth}{!}{%
$
     \begin{aligned}
        &\|\gT\omega - \gT u\|_\infty \\
        =& \max_{s',a'} \abs{(\gT\omega - \gT u)(s',a')} = \max_{s',a'} \abs{\gT\omega(s',a') - \gT u(s',a')} \\
        =& \max_{s',a'} \left| \frac{\gamma\sum_{s,a} \pi(a' \given s') P^*(s' \given s,a) \dbgtrue(s,a) (\omega(s,a) - u(s,a))  }{\dbgtrue(s',a')} \right| \\
        \leq& \max_{s',a'}  \frac{\gamma\sum_{s,a} \pi(a' \given s') P^*(s' \given s,a) \dbgtrue(s,a) |\omega(s,a) - u(s,a)|  }{\dbgtrue(s',a')} \\
        \leq& \max_{s',a'}  \frac{\gamma\sum_{s,a} \pi(a' \given s') P^*(s' \given s,a) \dbgtrue(s,a)}{\dbgtrue(s',a')} \cdot \|\omega - u\|_\infty  \\
        =& \left \{ \max_{s',a'}  \frac{\gamma\sum_{s,a} \pi(a' \given s') P^*(s' \given s,a) \dbgtrue(s,a)}{\gamma\sum_{s,a} \pi_b(a' \given s') P^*(s' \given s,a) \dbgtrue(s,a) + (1-\gamma) \mu_0(s') \pi_b(a'  \given  s')} \right \} \cdot \|\omega - u\|_\infty  \\
        =& \left \{ \max_{s',a'} \frac{\pi(a' \given s')}{\pi_b(a' \given s')} \cdot  \frac{\gamma\sum_{s,a} P^*(s' \given s,a) \dbgtrue(s,a)}{\gamma\sum_{s,a}  P^*(s' \given s,a) \dbgtrue(s,a) + (1-\gamma) \mu_0(s')} \right \} \cdot \|\omega - u\|_\infty  \\
        =& \left \{ \max_{s',a'} \frac{\pi(a' \given s')}{\pi_b(a' \given s')} \cdot  \underbrace{\left(1 - \frac{(1-\gamma) \mu_0(s')}{\gamma\sum_{s,a}  P^*(s' \given s,a) \dbgtrue(s,a) + (1-\gamma) \mu_0(s')} \right)}_{\triangleq\, c(s') < 1} \right \} \cdot \|\omega - u\|_\infty  \\
        =& \left \{ \max_{s',a'}  \frac{\pi(a' \given s')}{\pi_b(a' \given s')} \cdot c(s') \right \} \cdot  \|\omega - u \|_\infty \triangleq c \cdot  \|\omega - u \|_\infty.
    \end{aligned}
$
}
\end{equation}
If the current policy $\pi$ is close to the behavior policy $\pi_b$, in a sense that $\frac{\pi(a' \given s')}{\pi_b(a' \given s')}$ is close to $1$ on the entire finite state-action space, specifically,
\begin{equation*}
    \forall \, s',a', \quad \frac{\pi(a' \given s')}{\pi_b(a' \given s')} < \frac{1}{c(s')}.
\end{equation*}
Then $c = \max_{s',a'}  \left \{ \frac{\pi(a' \given s')}{\pi_b(a' \given s')} \cdot c(s') \right \} < 1$.
Plug into Eq. \eqref{eq:contraction_derive}, we have $\| \gT\omega - \gT u\|_\infty \leq c \cdot \|\omega - u \|_\infty$ for $c < 1$.
Thus $\gT$ is a $c$-contraction mapping under $\| \cdot \|_\infty$.
It follows that there exists a unique fixed point $\omega^*$ under the mapping $\gT$, such that $\omega^* = \gT \omega^*$.
Finally, we have
\begin{equation*}
    \| \omega_{k+1} - \omega^*\|_\infty = \| \gT \omega_k - \gT \omega^* \|_\infty \leq c \cdot \| \omega_k - \omega^* \|_\infty \leq \cdots \leq c^{k+1} \cdot \| \omega_0 - \omega^* \|_\infty \rightarrow 0, \text{ as } k \rightarrow \infty,
\end{equation*}
which shows that the iterate defined by $\gT$ converges geometrically.
\end{proof}

\section{Discussion on the choice of test function for MIW}
\label{sec:miw_diss}

We first discuss the original minimax optimization loss \citep{breakingcurse2018,gendice2020,uehara2020minimax} for learning the MIW $\omega(s,a)$.
When given an offline dataset $\mathcal{D}_{\mathrm{env}}$, the objective is
\begin{equation}\label{equ:wloss}\textstyle 
\resizebox{.93\textwidth}{!}{%
$
    \min_{\omega \in W}\max_{f \in \mathcal{F}}\cbr{
    (1-\gamma)\E_{\mu_0,\pi}[f(s,a)] + 
      \E_{(s,a, s^\prime)\sim \dbgtrue, a^\prime \sim \pi(\cdot \given s^\prime)}\left[ \gamma \omega(s,a)f(s^\prime, a^\prime) - \omega(s,a)f(s,a)\right]
    }, 
      $
}
\end{equation}
where $f \in \mathcal{F}$ is a test function similar to the discriminator in adversarial training.

A number of prior works on learning $\omega(s,a)$ (\eg~DualDICE \citep{dualdice2019}, GenDICE \citep{gendice2020}) can be viewed as variants of the above minimax objective, with some additional regularization terms on $\omega(s,a)$. 
For example, there is an additional $\frac{1}{2}\omega^2$ term in DualDICE to penalize $\omega$ so that the resulting MIW values $\omega(s,a)$ will not be too large.
This term thus serves as a regularization.

However, the optimization process for $\omega(s,a)$ in Eq.~\eqref{equ:wloss}  can be unstable due to the following two reasons: 
\textbf{(1)} the minimax optimization  is itself a challenging problem, especially for neural-network-based function approximator; 
\textbf{(2)} when we fix $f(s,a)$ and optimize $\omega(s,a)$, 
the gradients for $\omega$ that come from $\gamma \omega(s,a)f(s^\prime, a^\prime)$ and $\omega(s,a)f(s,a)$ are close to each other, especially when $\gamma$ is close to 1 (\eg, $0.99$). This makes the overall gradient for $\omega$ small and unstable, especially when we use stochastic gradient descent to optimize $\omega$.
The empirical distributions of the MIW from DualDICE and GenDICE in Fig.~\ref{fig:ablation_w} illustrate the optimization difficulty of these two methods. 
Therefore, we wish to get rid of the minimax optimization and 
stabilize the training process. 

Another closely related work that also tries to address the aforementioned issues is VPM \citep{vpm2020}, where the authors use the MIW $\omega$ itself as the test function $f$, based on the observation that the optimum test function under the $f$-divergence is the MIW; and use the variational power iteration to solve for $\omega$. 
We refer to \citet{vpm2020} for the detailed derivations. 
Empirically, we can observe in Fig.~\ref{fig:abla_distplot_vpm} that the distribution of $\omega(s,a)$ obtained by VPM  tends to be more stable than those obtained by DualDICE and GenDICE. 

Our choice of the test function is motivated by the following constrained optimization formulation of off-policy evaluation \citep{tang2019doubly,uehara2020minimax}:
\begin{equation*}\textstyle
    \begin{split}
        &\min_{Q}~~ (1-\gamma)\E_{s \sim \mu_{0}, a\sim \pi(\cdot \given  s)}\left[Q(s,a)\right] \\
    &\mathrm{s.t.}~~~~Q(s,a) \geq r(s,a) + \gamma \E_{s^\prime \sim P(\cdot \given s, a), a^\prime \sim \pi(\cdot \given s^\prime)}\left[Q(s^\prime, a^\prime)\right],\quad\forall\,(s,a) \in \sS \times \sA \,.
    \end{split}
\end{equation*}
One can show that $Q_\pi^{P^*}(s,a)$ is the optimal solution of the above (primal) optimization problem.
The dual form of this optimization problem can be reformulated as 
\begin{equation*}\textstyle
    \begin{split}
        &~~~~~~~~~\max_{d: \, \sS \times \sA \rightarrow \R_+}~~ \E_{d}[r(s,a)] \\
    &\mathrm{s.t.}~~~~d(s^\prime,a^\prime) = (1-\gamma)\mu_0(s^\prime)\pi(a^\prime \given s^\prime) + \gamma \sum_{s,a}d(s,a)P(s^\prime \given  s,a)\pi(a^\prime \given  s^\prime),\quad \forall\, (s^\prime, a^\prime)\in \sS\times \sA\,.
    \end{split}
\end{equation*}
One can show that the stationary distribution $d^{P^*}_{\pi, \gamma}$ is the solution to the above dual problem.

For the dual problem, we introduce the Lagrange multiplier $Q(s,a)$ for each $(s,a)$ tuple, and the Lagrangian can be written as  
\begin{align*}
    L(d, Q):= \E_{d}[r(s,a)] + (1-\gamma)\E_{\mu_0,\pi}[Q(s,a)] + \E_{d, P, \pi}[ \gamma Q(s^\prime, a^\prime) - Q(s,a)]\,.
\end{align*}
From the previous observations, the \textit{optimal Lagrange multiplier} for this Lagrangian is $Q_\pi^{P^*}(s,a)$, which is the solution to the primal problem. 
Changing $d$ to the MIW $\omega$, we have 
\begin{equation}\label{eq:larg_w}
    \resizebox{.93\textwidth}{!}{%
$
     \begin{aligned}
        L(\omega, Q)= &\underbrace{\E_{(s,a)\sim \dbgtrue}[\omega(s,a)r(s,a)]}_{\circled{1}:~~\text{Estimator for}~J(\pi, P^*)~\text{via}~\omega(s,a)} \\
    + & \underbrace{(1-\gamma)\E_{\mu_0,\pi}[Q(s,a)] + \E_{(s,a, s^\prime)\sim \dbgtrue, a^\prime \sim \pi(\cdot \given s^\prime)}[ \gamma \omega(s,a)Q(s^\prime, a^\prime) - \omega(s,a)Q(s,a)]}_{\circled{2}:~~\text{Loss for } \omega(s,a)~\text{that ideally should be }0}\,,
    \end{aligned}
$
}
\end{equation}
where the first term $\circled{1}$ estimates the expected return $J(\pi, P^*)$ via $\omega(s,a)$, which is accurate if $\omega$ is the true density ratio.
The second term $\circled{2}$ of Eq.~\eqref{eq:larg_w} is a loss for $\omega(s,a)$ as in Eq.~\eqref{equ:wloss}, except that we replace the test function $f(s,a)$ with $Q(s,a)$. 
With the \textit{optimal Lagrange multiplier} $Q_\pi^{P^*}(s,a)$, we can get rid of the inner minimization of the original maximin problem. 
In practice, since $Q_\pi^{P^*}$ is unknown, one may directly set $Q$ to be the ``estimated optimal multiplier'' $Q_\pi^{\widehat P}$ and optimize $\omega$ solely.

From Eq.~\eqref{eq:larg_w}, the MIW $\omega$ can be optimized via two alternative approaches.
\textbf{(1)} We can directly optimize $\omega$ with $L(\omega, Q_\pi^{\widehat P})$, whose optimum should be close to $J(\pi, P^*)$.
\textbf{(2)} We can optimize $\omega$ \wrt~$\circled{2}$, because we know that if $\circled{2}$ is close to zero, $\omega$ is close to the true MIW, and the resulting $L(\omega, Q_\pi^{\widehat P})$ is again close to $J(\pi, P^*)$ based on the term $\circled{1}$.
In theory the main difference of these two approaches is that we may obtain a more accurate estimate for $J(\pi, P^*)$ using approach \textbf{(1)} since we optimize $\omega$ on both the first and the second terms.
However, since the true reward function $r(s,a)$ is unknown, when we optimize $\omega$ using approach \textbf{(1)}, the optimization process may be unstable due to the approximation error for $r(s,a)$. 
% as observed in previous work that directly perform minimax optimization on $L(\omega, Q)$ \citep{tang2019doubly,uehara2020minimax}. 
By contrast, for approach \textbf{(2)}, we know that $\omega$ is accurate when $\circled{2}$ is close to zero, thus we can leverage the MSE loss to optimize $\omega$ so that $\circled{2}$ is shrunk towards zero. 
This approach is similar to the supervised learning and can be more stable in practice. 
Since our goal is to obtain a good estimate of the MIW $\omega$, not the numerical value of $J(\pi, P^*)$, we therefore choose approach \textbf{(2)} that optimize $\omega$ \wrt~$\circled{2}$.

Another intuitive motivation for using $Q_\pi^{P^*}$ as the test function is that the minimax loss Eq.~\eqref{equ:wloss} for optimizing $\omega$ is based on a \textit{saddle-point optimization}, and $f\in \mathcal{F}$ is the \textit{dual variable} for the MIW $\omega$.
From the previous discussion, MIW and the action-value function have some primal-dual relationship, and $Q_\pi^{P^*}$ is the optimal ``dual'' variable for the off-policy evaluation problem \wrt~the MIW $\omega$.
Since the dual variable $f$ in Eq.~\eqref{equ:wloss} is indeed some action-value function \citep{dualdice2019}, we may set $f$ to be this optimal ``dual'' of $\omega$, which will lead to the choice of $Q_\pi^{P^*}$ as the test function.
% We further use the design of target network $\omega'(s,a)$ for training stability. 

\myparagraph{Remark.} 
Note that both our approach and VPM choose the test function based on some mathematical relationships between the test function and the MIW $\omega$, 
so that the objective for $\omega$ could potentially be easier to optimize.
Both methods are not the \textit{best theoretical way} to optimize $\omega$. 
The best theoretical way is simply using the saddle-point-optimization based methods such as DualDICE, GenDICE, or a direct minimax optimization on $L(\omega, Q)$ in Eq.~\eqref{eq:larg_w}. 
However, as we discussed above, there are some practical difficulties for the saddle-point-optimization based approaches. 
As we observed in Fig.~\ref{fig:ablation_w}, the empirical results indeed show that our approach and VPM are more stable than DualDICE and GenDICE. 
Further, it is feasible to use other functions as the test function, such as the reward function or any other initialized neural networks. 
However, these choices may not have nice mathematical properties, such as being the fixed-point solution or some primal-dual relationship, and thus the resulting objective may not be easier to optimize. 
Table~\ref{table:additional_ablation} in Appendix~\ref{sec:additional_tables} contains an ablation study where we use the reward function as the test function to optimize $\omega$.
Indeed, this variant generally performs worse than our main method that uses the Q-function as the test function.

\section{Using value function as the discriminator of model training}\label{sec:q_dis}

Eq.~\eqref{eq:eval_error_policy_model} of the proof of Theorem~\ref{thm:eval_error_upper_bound} motivates us to use the value function $V_{\pi}^{\widehat P}(\cdot)$ as the discriminator of the model training.
Specifically, the model training objective is
\begin{equation}\label{eq:q_dis_obj}
    \arg\min_{\widehat P \in \gP} \mathbb{E}_{(s,a)\sim \dbgtrue}\left\{ \omega(s,a) \cdot \left\vert\mathbb{E}_{s^{\prime}\sim P^{*}(\cdot \given s,a)}\left[V_{\pi}^{\widehat{P}}(s^{\prime})\right]-\mathbb{E}_{{s'}\sim\widehat{P}(\cdot \given s,a)}\left[V_{\pi}^{\widehat{P}}({s'})\right]\right\vert\right\},
\end{equation}
where the discriminator $V_{\pi}^{\widehat P}$ can be implemented by the current estimate of the action-value function, which is treated as fixed during the model training.
We implement this idea as a variant where the model is trained by the weighed objective Eq.~\eqref{eq:q_dis_obj} that uses the value-function estimate to discriminate the predicted transition from the real.
The reward function $\hat r$ is still estimated by the weighted-MLE objective.
The model is initialized by the MLE loss.
Other technical details follow the main algorithm.
We dubbed this variant as ``V-Dis''.

Table~\ref{table:ablation_q_dis} compares this variant with our main method.
We see that V-Dis underperforms our main method on almost all dataset, often by a large margin. The inferior results of V-Dis may be related to the coupled effect on model training from the inaccurate value-function estimation.
In Eq.~\eqref{eq:q_dis_obj}, this inaccuracy can affect both the value difference (the $\abs{\cdot}$ term) and the MIW $\omega$.
Our fix of $V_\pi^{\widehat P}$ during the model training process may also affect the performance. 
Future work can investigate how to train $V_\pi^{\widehat P}$ together with $\widehat P$ in the model estimation.

% \begin{wrapfigure}{R}{0.5\textwidth}
% \begin{minipage}{0.5\textwidth}
% \vspace{-8em}

% \vspace{-5em}
% \end{minipage}
% \end{wrapfigure}

\begin{table}[t] 
%\captionsetup{font=small}
\caption{
\small Normalized returns for experiments on the D4RL tasks.
Mean and standard deviation across five random seeds are reported.
``V-Dis'' denotes the variant where the model is trained by the value-function-discriminated weighted objective Eq.~\eqref{eq:q_dis_obj}.
The results of our main method in \Secref{sec:exp_main} is reported in column Main.
Here, ``hcheetah'' denotes ``halfcheetah,'' ``med'' denotes ``medium,'' ``rep'' denotes ``replay,'' and ``exp'' denotes ``expert.''} 
\label{table:ablation_q_dis} 

\centering 
\renewcommand{\arraystretch}{1.1}
%\resizebox{1.\textwidth}{!}{
\setlength{\tabcolsep}{35pt}
\begin{tabular}{lcc}
\toprule
                 Task Name &                              V-Dis &                              Main \\
\midrule
              maze2d-umaze &   36.7 $\pm$ {\footnotesize 28.3} &    55.8 $\pm$ {\footnotesize 3.9} \\
             maze2d-med &   36.9 $\pm$ {\footnotesize 19.0} &  107.2 $\pm$ {\footnotesize 45.0} \\
              maze2d-large &   23.9 $\pm$ {\footnotesize 17.6} &  180.0 $\pm$ {\footnotesize 39.3} \\
        hcheetah-med &   38.3 $\pm$ {\footnotesize 17.1} &    51.7 $\pm$ {\footnotesize 0.4} \\
           walker2d-med &   46.9 $\pm$ {\footnotesize 28.3} &    83.1 $\pm$ {\footnotesize 1.8} \\
             hopper-med &    58.4 $\pm$ {\footnotesize 7.4} &    58.9 $\pm$ {\footnotesize 7.9} \\
 hcheetah-med-rep &   21.3 $\pm$ {\footnotesize 13.9} &    44.6 $\pm$ {\footnotesize 0.7} \\
    walker2d-med-rep &   16.3 $\pm$ {\footnotesize 12.9} &    81.5 $\pm$ {\footnotesize 3.0} \\
      hopper-med-rep &   69.0 $\pm$ {\footnotesize 19.6} &    91.1 $\pm$ {\footnotesize 9.5} \\
 hcheetah-med-exp &    39.7 $\pm$ {\footnotesize 4.2} &    90.2 $\pm$ {\footnotesize 2.0} \\
    walker2d-med-exp &   51.3 $\pm$ {\footnotesize 29.4} &   107.7 $\pm$ {\footnotesize 2.1} \\
      hopper-med-exp &    24.0 $\pm$ {\footnotesize 8.0} &   76.0 $\pm$ {\footnotesize 15.7} \\
                 pen-human &   11.4 $\pm$ {\footnotesize 10.0} &   20.6 $\pm$ {\footnotesize 10.7} \\
                pen-cloned &   36.3 $\pm$ {\footnotesize 10.5} &   57.4 $\pm$ {\footnotesize 19.2} \\
                pen-exp &  131.3 $\pm$ {\footnotesize 16.7} &   138.3 $\pm$ {\footnotesize 6.6} \\
               door-exp &   60.4 $\pm$ {\footnotesize 33.0} &    96.3 $\pm$ {\footnotesize 9.8} \\
               \midrule
             Average Score &                              43.9 &                              83.8 \\
\bottomrule
\end{tabular}
%}
\end{table}

\section{Technical details} \label{sec:exp_details}

% \subsection{Details for the off-policy evaluation experiment} \label{sec:ope_details}

\subsection{Details for the main algorithm} \label{sec:algo_details}
Our main algorithm consists of three major components: model training (Appendix~\ref{sec:algo_details_model}), the marginal importance weight training (Appendix~\ref{sec:algo_details_td_dice}), and RL policy training via actor-critic algorithm (Appendix~\ref{sec:algo_details_ac}).

% reward scaling
The non-zero assumption of $Q_\pi^{\widehat P}$ in the derivation of Eq. \eqref{eq:dr_target_final} motivates us to scale the rewards in the offline dataset to be roughly within $(0,1]$, similar to \citet{lapo2022}.
Specifically, before starting our main algorithm, we normalize the rewards $r_i$'s in the offline dataset via $(r_i - r_{\min} + 0.001)/(r_{\max} - r_{\min})$, where $r_{\min}$ and $r_{\max}$ respectively denote the minimum and maximum reward in the offline dataset.

\subsubsection{Model training} \label{sec:algo_details_model}

We follow the literature \citep{pets2018, mbpo2019, mopo2020,wmopo2021} to assume no prior knowledge about the reward function and thus use neural network to approximate transition dynamic and the reward function.
Our model is constructed as an ensemble of Gaussian probabilistic networks .
Except for the weighted loss function, we use the same model architecture, ensemble size (=7), number of elite model (=5), train-test set split, optimization scheme, elite-model selection criterion, and sampling method as in \citet{mopo2020}.
To save computation, we define each epoch of model training as $1000$ mini-batch gradient steps, instead of the original {\tt n\_train / batch\_size}.

As in \citet{mopo2020}, the input to our dynamic model is $\br{(s,a) - \mu_{(s,a)}} / \sigma_{(s,a)}$, where $\mu_{(s,a)}$ and $\sigma_{(s,a)}$ are respectively coordinate-wise mean and standard deviation of $(s,a)$ in the offline dataset.
We follow \citet{morel2020} and \citet{bremen2021} to define the learning target as $\br{\br{r, \Delta s} - \mu_{\br{r, \Delta s}}} / \sigma_{\br{r, \Delta s}}$, where $\Delta s$ denote $s' - s$ ,  $\mu_{\br{r, \Delta s}}$ and $\sigma_{\br{r, \Delta s}}$ denote the coordinate-wise mean and standard deviation of $\br{r, \Delta s}$ in the offline dataset.
As in prior work using Gaussian probabilistic ensemble on model-based RL \citep{pets2018,  mbpo2019, mopo2020, mabe2021, combo2021}, we use a double-head architecture for our dynamic model, where the two output heads represent the mean and log-standard-deviation of the normal distribution of the predicted output, respectively.
We augment the maximum likelihood loss in \citet{mopo2020} as the weighted maximum likelihood loss, independently for each model in the ensemble.
% MLE and WMLE for model training 
Specifically, the dynamic model is initially trained using un-weighted maximum likelihood loss. 
We then employ the warm-start strategy to start each subsequent weighted re-training from the current dynamic model, with the weights $\omega(s,a)$ from the latest marginal importance weight network.
% self-normalization
To unify the learning rate across each run of model (re-)training, the weights $\omega(s,a)$'s are normalized so that their mean is $1$ across the offline dataset.

% augmented terminal function
To improve training stability, we are motivate by \citet{wmopo2021} to heuristically augment the environmental termination function by {(1)} $\mathrm{any}(|\hat{s}'| > 2 \cdot \max_{s, s' \in \denv}\cbr{|s|, |s'|})$ where all operations are coordinate-wise, \ie, whether any coordinate of the predicted next state is outside of this relaxed observation-range in the offline dataset; and {(2)} $|\hat{r}(s,a)| > r_{\mathrm{range}} \triangleq \max\br{|r_{\min} - 10 \cdot \sigma_r|, |r_{\max} + 10 \cdot \sigma_r|)}$ where  $r_{\min}, r_{\max}, \sigma_r$ are the minimum, maximum, and standard deviation of rewards in the offline dataset, \ie, whether the predicted reward is outside of this relaxed reward-range in the offline dataset.
To penalize the policy for visiting out-of-distribution state-action pair, we modify $\hat{r}(s,a)$ to be $- r_{\mathrm{range}}$ if either (1) or (2) is violated.
No modification for $\hat{r}(s,a)$ is performed if only the environmental termination is triggered.

\subsubsection{Marginal importance weight training} \label{sec:algo_details_td_dice}
We follow the OPE literature \citep{breakingcurse2018,algaedice2019,gendice2020} to assume that the initial-state distribution $\mu_0$ is known.
In the implementation, we augment $\denv$ with $10^5$ samples from the ground-truth $\mu_0$.
To enhance training stability, we are motivated by the double Q-learning \citep{doubleq2010} to use a target marginal importance weight network $\omega'(s,a)$ and a target critic network $Q'(s,a)$ in the implementation of MIW training.
We employ the warm-start strategy so that each retraining of the marginal importance weight starts from the current $\omega(s,a)$. 
On each run of the weight-retraining, $Q'$ is initialized as the current critic target $Q_{\vtheta'}$.
Since the weights $\omega(s,a)$'s are normalized to have mean $1$ in the weighted re-training of the dynamic model, to narrow the gap between the learning and the application of the weights, we add a constraint into the training of $\omega$ that its mean  on each mini-batch is upper bounded by some constant $g_{\mathrm{constraint}}$.
Specifically, on each mini-batch gradient step, MIW is trained by the following constraint optimization
\begin{equation}\label{eq:td_dice_minibatch}
    \resizebox{0.93\textwidth}{!}{%
    $
    \begin{aligned}
         \arg\min_{\omega}& \; \br{\frac{1}{|{\gB}|} \sum_{\br{s_i, a_i} \in {\gB}} \omega\br{s_i, a_i} \cdot Q_\vtheta\br{s_i, a_i} - y}^2, \\ 
         \mathrm{s.t.}~~~& \; \frac{1}{|{\gB}|} \sum_{\br{s_i, a_i} \in {\gB}} \omega\br{s_i, a_i} \leq g_{\mathrm{constraint}}, \\
         \text{where } & \; y = \gamma \cdot \frac{1}{|\gB|}\sum_{~{\br{s_i, a_i, s'_i} \in \gB}\atop {a'_i\sim \pi_\vphi\br{\cdot \given s'_i}}} \omega'\br{s_i, a_i} \cdot Q'\br{s'_i, a'_i} + (1-\gamma) \cdot \frac{1}{|\gB_{\mathrm{init}}|}\sum_{~{s_0 \in \gB_{\mathrm{init}}}\atop {a_0\sim \pi_\vphi\br{\cdot \given s_0}}} Q'\br{s_0, a_0},  
   \end{aligned}
$%
}
\end{equation}
where $g_{\mathrm{constraint}}$ is conveniently chosen as $10$ and $\gB, \gB_{\mathrm{init}} \sim \denv$. 
To optimize Eq.~\eqref{eq:td_dice_minibatch}, we use the constraint optimization method in \citet{barrierlexico2021} with base optimizer Adam \citep{adam2014}.
For training stability, the gradient norm is clipped to be bounded by $1$.

As in double Q-learning, we update $Q'$ towards $Q_\vtheta$ and $\omega'$ towards $\omega$ via exponential moving average with rate $0.01$ after each gradient step.
We use batch sizes $|\gB| = 1024$ and $|\gB_{\mathrm{init}}| = 2048$ to train the marginal importance weight model $\omega(s,a)$.

\subsubsection{Actor-critic training} \label{sec:algo_details_ac}
Our policy training consists of the following six parts.

\paragraph{Generate synthetic data using dynamic model.} 
We follow \citet{mbpo2019} and \citet{mopo2020} to perform $h$-step rollouts branching from the offline dataset $\denv$, using the learned dynamic $\widehat P, \hat{r}$, and the current policy $\pi_\vphi$.
The generated data is added to a separate replay buffer $\dmodel$.

To save computation, we generate synthetic data every {\tt rollout\_generation\_freq} iterations.

\paragraph{Critic training.} 
Sample mini-batch $\gB \sim \gD = f \cdot \denv + (1 - f)\cdot \dmodel$, with $f \in [0,1]$.
For each $s' \in \gB$, sample one corresponding actions $a' \sim \pi_{\vphi'}(\cdot \given s')$.
Calculate the estimate of the action-value of the next state-action pair as
\begin{equation*}
         \widetilde{Q}\br{s, a} \triangleq c \min_{j=1,2}Q_{\vtheta'_j}\br{s', a'} + (1-c) \max_{j=1,2}Q_{\vtheta'_j}\br{s', a'}.
\end{equation*}
Calculate the critic-learning target defined as
\begin{equation} \label{eq:critic_target}
    \widetilde{Q}\br{s, a} \leftarrow r(s,a) + \gamma \cdot \widetilde{Q}\br{s, a} \cdot \br{\abs{\widetilde{Q}\br{s, a}} < 2000},
\end{equation}
where $2000$ is conveniently chosen to enhance the termination condition stored in $\gB$ for training stability.
Then minimize the critic loss with respect to the double critic networks, over $\br{s, a} \in \gB$, with learning rate $\eta_\vtheta$,
\begin{equation*}
    \forall\, j = 1,2, \; \arg\min_{\vtheta_j}\frac{1}{\abs{\gB}}\sum_{\br{s, a} \in \gB} \mathrm{Huber}\br{ Q_{\vtheta_j}\br{s, a}, \widetilde{Q}(s, a) },
\end{equation*}
where $\mathrm{Huber}(\cdot)$ denote the Huber-loss with threshold conveniently chosen as $500$.
Finally, the norm of the gradient for critic update is clipped to be bounded by $0.1$ for training stability.

As a remark, the mentioned enhanced termination condition, Huber loss and gradient-norm clipping are engineering designs for stable training in combating for not knowing the true reward function.
We use a unified setting for these designs across all tested datasets in all experiments, and thus a per-dataset tuning may further improve our results.

\paragraph{Discriminator training.}
The ``fake'' samples $\gB_{\mathrm{fake}}$ for discriminator training is formed by the following steps: 
first, sample $\abs{\gB}$ states $s \sim \gD$;
second, for each $s$ get one corresponding action $a\sim \pi_\vphi\br{\cdot \given s}$;
third, for each $(s,a)$ get an estimate of the next state $s'$ using $\widehat P$ as $s' \sim \widehat P(\cdot \given s, a)$, with terminal states removed;
fourth, sample one next action $a'$ for each next state $s'$ via $a' \sim \pi_\vphi\br{\cdot \given s'}$. 
The ``fake'' samples $\gB_{\mathrm{fake}}$ is subsequently formed by
\begin{equation*}
\gB_{\mathrm{fake}} \triangleq     
\begin{bmatrix}
    (s\mbox{ }, & a\mbox{ }) \\
    (s', & a')
    \end{bmatrix}.
\end{equation*}
Note that terminal states are removed since by definition, no action choices are needed on them.
The ``true'' samples $\gB_{\mathrm{true}}$ consists of $\abs{\gB_{\mathrm{fake}}}$ state-action samples from $\denv$.

The discriminator $D_\vpsi$ is optimized with learning rate $\eta_\vpsi$ as
\begin{equation*}
    \arg\max_{\vpsi} \frac{1}{\abs{\gB_{\mathrm{true}}}} \sum_{(s,a)\sim \gB_{\mathrm{true}}}\log D_\vpsi(s,a) + \frac{1}{\abs{\gB_{\mathrm{fake}}}} \sum_{(s,a)\sim \gB_{\mathrm{fake}}} \log\br{1-D_\vpsi(s,a)}.
\end{equation*}

\paragraph{Actor training.}
The policy is updated once every $k$ updates of the critic and the discriminator.
Using the ``fake'' samples $\gB_{\mathrm{fake}}$ and the discriminator $D_\vpsi$, the generator loss for actor training is 
\begin{equation*}
    \gL_g(\vphi) = \frac{1}{\abs{\gB_{\mathrm{fake}}}} \sum_{(s,a) \in \gB_{\mathrm{fake}}}\sbr{\log\br{1-D_\vpsi\br{s,a}}}.
\end{equation*}
The policy is optimized with learning rate $\eta_\vphi$ as
\begin{equation*}
    \arg\min_{\vphi} -\lambda \cdot\frac{1}{\abs{\gB}} \sum_{s \in \gB, a \sim \pi_\vphi(\cdot \given s)}\sbr{\min_{j=1,2} Q_{\vtheta_j}(s, a)} + \gL_g(\vphi),
\end{equation*}
where in all tested datasets the regularization coefficient $\lambda \triangleq 10/Q_{avg}$ with soft-updated $Q_{avg}$ and the penalty coefficient $10$ conveniently chosen, similar to \citet{td3bc2021}.

\paragraph{Soft updates.}
We follow \citet{td32018} to soft-update the target-network parameter and the $Q_{avg}$ value,
\begin{equation*}
    \begin{split}
        \vphi' & \leftarrow \beta \vphi + (1-\beta) \vphi', \\
        \vtheta'_j &\leftarrow \beta \vtheta_j + (1-\beta)\vtheta'_j, \quad \forall\, j=1,2, \\
        Q_{avg} &\leftarrow \beta \frac{1}{\abs{\gB}} \sum_{s \in \gB, a \sim \pi_\vphi(\cdot \given s)} \abs{\min_{j=1,2} Q_{\vtheta_j}\br{s, a}} + (1 - \beta) Q_{avg}.
    \end{split}
\end{equation*}
Soft-updates are performed after each iteration.

\paragraph{Warm-start step.} 
We follow \citet{cql2020} and \citet{idac2020} to warm-start our policy training in the first $N_{\mathrm{warm}}$ epochs.
In this step, the policy is optimized \wrt\mbox{} the generator loss $\gL_g(\vphi)$ only, \ie, $\arg\min_\vphi \gL_g(\vphi)$.
The learning rate $\eta_\vphi$ in the warm-start step is the same as the normal training step.

\subsection{Details for the continuous-control experiment} \label{sec:rl_exp_details}

\paragraph{Datasets.}
We evaluate algorithms on the continuous control tasks in the D4RL benchmark \citep{fu2021d4rl}.
Due to limited computational resources, we select therein a representative set of datasets.
Specifically, 
\textbf{(1)} we select the ``medium-expert,'' ``medium-replay,'' and ``medium'' datasets for the Hopper, HalfCheetah, and Walker2d tasks in the Gym-MuJoCo domain, which are commonly used benchmarks in prior work \citep{bcq2019, bear2019, brac2019, cql2020}.
As in recent literature \citep{mabe2021,decisiontrans2021,iql2021}, we do not test on the ``random'' and ``expert'' datasets, as they are considered as less practical \citep{bremen2021} and can be simply solved by directly applying standard off-policy RL algorithms \citep{rem2020} and behavior cloning on the offline dataset.
We note that a data-quality agnostic setting may not align with practical offline RL applications, \ie, one typically should know the quality of the offline datasets, \textit{e.g.}, whether it is collected by random policy or by experts.
\textbf{(2)} Apart from the Gym-MuJoCo domain, we further consider the Maze2D domain of tasks\footnote{We use the tasks ``maze2d-umaze,'' ``maze2d-medium,'' and ``maze2d-large.''} for the non-Markovian data-collecting policy and the Adroit tasks\footnote{We use the tasks ``pen-human,'' ``pen-cloned,'' ``pen-expert,'' and ``door-expert.''} \citep{adroit2018} for their high dimensionality and sparse rewards.

\paragraph{Evaluation protocol.}
Our algorithm is trained for $1000$ epochs, where each epoch consists of $1000$ mini-batch gradient descent steps.
After each epoch of training, a rollout of $10$ episodes is performed.
Both our algorithm and baselines are run under five random seeds $\cbr{0,1,2,3,4}$, and we report the mean and standard deviation of the final rollouts across the five seeds.
We follow \citet{fu2021d4rl} to rerun the baselines under the recommended hyperparameter setting, including dataset-specific hyperparameters if available.

\paragraph{Results of CQL.} \label{sec:tech_cql}
We run CQL using the implementation suggested by \citet{iql2021}, which only provide hyperparameters for the Gym-MuJoCo and AntMaze domains.
For the Maze2D and Adroit datasets, we run both sets of recommended hyperparameters and per-dataset select the better result to report.
We note that hyperparameter settings for the Maze2D and Adroit datasets are also not provided in the original CQL implementation.
An in-depth tuning of CQL on these two domains of datasets is beyond the scope of this paper.

\paragraph{Implicit policy implementation.}
We follow \citet{samplegenerative2016} to choose the noise distribution as the isotropic normal, \ie, $p_z\br{z} \triangleq \gN\br{\vzero, \sigma_{\mathrm{noise}}^2 \mI}$, where the default setting of the dimension of $z$ is $\mathrm{dim}\br{z} = {\tt min(10, \, state\_dim // 2)}$ and of the noise standard deviation is $\sigma_{\mathrm{noise}} = 1$.
In the forward pass of the implicit policy, for a given state $s$, a noise sample $z \sim p_z(z)$ is first independently drawn.
Then $s$ is concatenated with $z$ and the resulting $\sbr{s, z}$ is inputted into the (deterministic) policy network to sample stochastic action.

\paragraph{Terminal states.}
In practice, the episodes in the offline dataset have finite horizon. 
Thus, special treatment is needed for those terminal states in calculating the Bellman update target.
We combine common practice \citep{dqn2013,rlintro2018} with our critic update target Eq.~\eqref{eq:critic_target} as 
\begin{equation*}
    \widetilde{Q}\br{s, a} = \begin{cases}
    r(s,a) + \gamma \cdot \widetilde{Q}\br{s, a} \cdot \br{\abs{\widetilde{Q}\br{s, a}} < 2000} & \text{ if $s'$ is a non-terminal state} \\
    r(s,a) & \text{ if $s'$ is a terminal state}
    \end{cases},
\end{equation*}
since the action-value function is undefined on the terminal $s'$ due to no action choice therein.

Our network architecture is presented in Appendix \ref{sec:algo_details_gan}. 
To save computation, we use simple neural networks as in \citet{bcq2019}.
A fine-tuning of the noise distribution $p_z\br{z}$, the network architecture, the optimizer hyperparameters, \etc, is left for future work.

\subsubsection{Details for the implementation of AMPL}  \label{sec:algo_details_gan}

With the training techniques for GAN suggested by previous work in generative modeling, we are able to stably and effectively train the GAN structure on data with moderate dimension, \eg, the D4RL datasets we consider.
In this paper, we adopt the following techniques from the literature.
\begin{itemize}[leftmargin=*]
    \item Following \citet{gan2014}, $\pi_\vphi$ is trained to maximize
    $
        \E_{(s,a) \in \gB_{\mathrm{fake}}} \sbr{\log\br{D_\vpsi(s,a)}}.
    $
    \item We are motivated by \citet{dcgan2016} to use LeakyReLU activation function in both the generator and the discriminator, with {\tt negative\_slope} being the default value $0.01$.
    \item For stable training, we follow \citet{dcgan2016} to use a reduced momentum term $\beta_1 = 0.4$ in the Adam optimizers of the policy and the discriminator networks; and to use learning rates $\eta_\vphi = \eta_\vpsi = 2 \times 10^{-4}$.
    \item To avoid discriminator overfitting, we are motivated by \citet{improvetechgan2016} and \citet{gantutorial2017} to use one-sided label smoothing with soft and noisy labels.
    Specifically, random numbers in $[0.8, 1.0)$ are used as the labels for the ``true'' sample $\gB_{\mathrm{true}}$.
    No label smoothing is needed for the ``fake'' sample $\gB_{\mathrm{fake}}$, \ie, their labels are all $0$.
    \item Binary cross entropy loss between the labels and the discriminator outputs is used as the loss for the discriminator training in the underlying GAN structure.
    \item We are motivated by TD3 \citep{td32018} and GAN to update the policy $\pi_\vphi(\cdot \given s)$ once per $k$ updates of the critics and the discriminator. 
\end{itemize}

Table~\ref{table:gan_param} shows the shared hyperparameters for all datasets in our experiments.
\emph{Many of these hyperparameters follow the literature without further tuning}.
For example, we use $\eta_\vphi = \eta_\vpsi = 2 \times 10^{-4}$ as in \citet{dcgan2016}, $\eta_\vtheta = 3 \times 10^{-4}$ and $N_{\mathrm{warm}} = 40$ as in \citet{cql2020}, $c=0.75$ as in \citet{bcq2019}, policy frequency $k=2$ as in \citet{td32018}, $f=0.5$ as in \citet{combo2021} and model learning rate $\eta_{\widehat P} = 0.001$ as in \citet{mopo2020}. 
Unless specified, the shared and the non-shared hyperparameters are used throughout the main results and the ablation study.

\begin{table}[ht]
\captionsetup{font=small}
\caption{\small Shared hyperparameters for AMPL.} \label{table:gan_param}
\centering
\begin{tabular}{@{}ll@{}}
\toprule
Hyperparameter & Value \\ \midrule
Optimizer      & Adam  \\
Training iterations & $10^6$ \\
Training iterations for $\omega(s,a)$ & $10^5$ \\
Model retrain period & per $100$ epochs \\
Learning rate $\eta_\vtheta$ &  $3 \times 10^{-4}$  \\ 
Learning rate $\eta_\vphi$, $\eta_\vpsi$ & $2 \times 10^{-4}$      \\ 
Learning rate $\eta_{\widehat P}$ &  $1 \times 10^{-3}$  \\ 
Learning rate $\eta_\omega$ for $\omega(s,a)$ & $1 \times 10^{-6}$ \\
Penalty coefficient & $10$ \\
Batch size & $512$ (as in \citet{optidice2021}) \\
Discount factor $\gamma$ & $0.99$ \\
Target network update rate $\beta$ & $0.005$ \\
Weighting for clipped double Q-learning $c$ & $0.75$ \\
Noise distribution $p_z(z)$ & $\gN\br{\vzero, \sigma_{\mathrm{noise}}^2\mI}$ \\
Default value of $\sigma_{\mathrm{noise}}$ & $1$ \\
Policy frequency $k$ & $2$ \\
Rollout generation frequency & per $250$ iterations \\
Number of model-rollout samples per iteration & $128$ \\
Rollout retain epochs & $5$ \\
Real data percentage $f$ & $0.5$  \\
Warm start epochs $N_{\mathrm{warm}}$ & $40$ \\
Random seeds & $\cbr{0,1,2,3,4}$ \\
\bottomrule
\end{tabular}
\end{table}

Since the tested datasets possess diverse nature, we follow the common practice in offline model-based RL to perform gentle dataset-specific hyperparameter tuning. 
Details are discussed below.

% fixed_rollout_len
We are motivated by \citet{mbpo2019} and \citet{mopo2020} to consider the rollout horizon $h\in \cbr{1,3,5}$. 
We use $h=1$ for 
halfcheetah-medium, hopper-medium, hopper-medium-replay, maze2d-large, pen-expert, door-expert, pen-human;
$h=3$ for 
hopper-medium-expert, walker2d-medium, halfcheetah-medium-replay, maze2d-umaze, maze2d-medium;
and $h=5$ for
halfcheetah-medium-expert, walker2d-medium-expert, walker2d-medium-replay, pen-cloned.

% noise_dim
\citet{act2021} and \citet{zhang2021alignment} show that a larger noise dimension, \eg, $50$, can help learning a more flexible distribution, while a smaller noise dimension, \eg, $1$, can make the distribution leaning towards deterministic. 
Hence we use the default noise dimension for the tasks:
maze2d-umaze, door-expert; 
noise dimension $50$ for the tasks:
hopper-medium-expert, halfcheetah-medium-expert, walker2d-medium-expert, halfcheetah-medium, hopper-medium, walker2d-medium, halfcheetah-medium-replay, hopper-medium-replay, walker2d-medium-replay, maze2d-medium, maze2d-large, pen-expert;
and noise dimension $1$ for the tasks:
pen-cloned, pen-human.

% sigma_human
Due to the narrow data distribution, we use $\sigma_{\mathrm{noise}} = 0.001$ for datasets pen-human and pen-cloned to avoid training divergence.

We state below the network architectures of actor, critic, discriminator, and MIW in the implementation of AMPL.
Motivated by the clipped double Q-learning, two critic networks are used with the same architecture.

\begin{multicols}{2}% 2-column layout
  \begin{minipage}{0.45\textwidth}
    Actor
\begin{lstlisting}
Linear(state_dim + noise_dim, 400)
LeakyReLU
Linear(400, 300)
LeakyReLU
Linear(300, action_dim)
max_action $\times$ tanh
\end{lstlisting}
  \end{minipage}
  
\begin{minipage}{0.45\textwidth}
Critic 
\begin{lstlisting}
Linear(state_dim + action_dim, 400)
LeakyReLU
Linear(400, 300)
LeakyReLU
Linear(300, 1)
\end{lstlisting}
\end{minipage}
\end{multicols}

\begin{multicols}{2}
\begin{minipage}{0.45\textwidth}
Discriminator
\begin{lstlisting}
Linear(state_dim + action_dim, 400)
LeakyReLU
Linear(400, 300)
LeakyReLU
Linear(300, 1)
Sigmoid
\end{lstlisting}
\end{minipage}

\begin{minipage}{0.45\textwidth}
MIW
\begin{lstlisting}
Linear(state_dim + action_dim, 400)
LeakyReLU
Linear(400, 300)
LeakyReLU
Linear(300, 1)
(softplus($\cdot$ - $10^{-8}$) + $10^{-8}$).power($\alpha$)
\end{lstlisting}
\end{minipage}
\end{multicols}

% alpha
In the MIW network, $\alpha$ is the smoothing exponent to smooth the distribution of the weight estimates in the offline dataset, whose design is motivated by \citet{wmopo2021}. 
We fix $\alpha$ as $0.5$ on the MuJoCo datasets and as $0.2$ on the Maze2D and Adroit datasets.
The output of the MIW network prior to ${\tt power}(\alpha)$ is lower bounded by $10^{-8}$ to avoid gradient explode.
We follow \citet{coindice2020} to initialize the weights of last layer of the MIW network from $\mathrm{Uniform}\br{-0.003, 0.003}$, and bias as $0$.

\subsection{Details for the implementation of alternative MIW estimation methods
in \Secref{sec:exp_ablation}}\label{sec:dice_details}

We implement and train the alternative methods VPM, GenDICE, and DualDICE in \Secref{sec:exp_ablation} based on the corresponding papers and source codes.
To stabilize training, the design of target MIW network in \Secref{sec:algo_details_td_dice} is applied wherever applicable.
The target MIW network is hard-updated, as suggested by the corresponding papers.
The $\nu$ network in GenDICE and DualDICE has the same architecture as the critic network and has its own hard-updated target network.
We follow the official implementation of GenDICE and DualDICE to clip the gradient values by $1$.
% dualdice use f
For the convex function $f$ in DualDICE, we use $f(x) = \frac{2}{3}\abs{x}^{3/2}$, as suggested in its Section 5.3.

\section{Potential negative societal impacts}\label{sec:neg_impact}
Since offline RL has strong connection with the supervised learning, offline RL methods, including the proposed AMPL in this paper, are subject to the bias in the offline dataset.
The data-bias can be exploited by both the dynamic-model learning and the policy training.

\section{Computing resources}\label{sec:compting_resouces}
We ran all experiments on 4 NVIDIA Quadro RTX 5000 GPUs.

% \rebuttal{
% The second term in the RHS of the above equation can be viewed as the loss for $\omega$. And if the loss is zero, then $\omega$ is the true density ratio function. 
% As a result, we propose to use the optimal Lagrangian multiplier to optimize the density ratio $\omega(s,a)$, which is essentially the action-value function $Q^\pi(s,a)$. Further, we use the design of target network for training stabilization.} 

% 1. discuss why we choose q 
% 2. simplify the loss based on xx. 
% 3. admit that both vpm and our approach to choose test function is 

\end{document}